\documentclass{article}

\usepackage[final]{neurips_2022}
\usepackage[utf8]{inputenc} % allow utf-8 input
\usepackage[T1]{fontenc}    % use 8-bit T1 fonts
\usepackage{hyperref}       % hyperlinks
\usepackage{url}            % simple URL typesetting
\usepackage{booktabs}       % professional-quality tables
\usepackage{amsfonts}       % blackboard math symbols
\usepackage{nicefrac}       % compact symbols for 1/2, etc.
\usepackage{xcolor}         % colors

\usepackage{microtype}
\usepackage{graphicx}
\usepackage{subfigure}
\usepackage{amsmath}
\usepackage{amsthm}
\usepackage{mathtools}
\usepackage{natbib}
\usepackage[textsize=tiny]{todonotes}
\usepackage[capitalize,noabbrev]{cleveref}
\usepackage{algorithm}
\usepackage{algpseudocode}
\usepackage{tcolorbox}
\usepackage{bm}
\usepackage[shortlabels]{enumitem}
\usepackage{wrapfig}

\usepackage{hyperref}

\newcommand{\Pclass}{\mathsf{P}}

\newcommand{\NP}{\mathsf{NP}}
\newcommand{\MAXSNP}{\mathsf{MAX SNP}}

\newcommand{\R}{\mathbb{R}}

\newcommand{\MAXCUT}{\text{Max-Cut}}
\newcommand{\norm}[1]{||#1||}
\newcommand\inprod[1]{\langle #1 \rangle}
\newcommand{\func}{\circ}
\newcommand{\ACTL}{\text{diag}}
\newcommand{\act}{\sigma}

\newcommand{\dualv}{\lambda}
\newcommand{\dualvec}{\Lambda}
\newcommand{\trace}{\text{tr}}

\newcommand{\relu}{\text{ReLU}}
\newcommand{\obju}{\zeta}

\renewcommand{\cite}{\citep}

\theoremstyle{plain}
\newtheorem{theorem}{Theorem}[section]

\newtheorem{lemma}[theorem]{Lemma}

\theoremstyle{definition}

\theoremstyle{remark}
\newtheorem{remark}[theorem]{Remark}

\renewcommand{\paragraph}[1]{\vspace{.05in}\noindent\textbf{{#1}.~~}}

\title{A Quantitative Geometric Approach to Neural-Network Smoothness}

\author{%
  Zi Wang
    \\
  University of Wisconsin-Madison\\
  \texttt{zw@cs.wisc.edu} \\
   \And
  Gautam Prakriya \\
  The Chinese University of Hong Kong \\
   \texttt{gautamprakriya@gmail.com} \\
   \AND
   Somesh Jha \\
  University of Wisconsin-Madison\\
  \texttt{jha@cs.wisc.edu} \\
}

\begin{document}

\maketitle

\begin{abstract}
Fast and precise Lipschitz constant estimation of neural networks is an important task for deep learning. Researchers have recently found an intrinsic trade-off between the accuracy and smoothness of neural networks, so training a network with a loose Lipschitz constant estimation imposes a strong regularization, and can hurt the model accuracy significantly. In this work, we provide a unified theoretical framework, a quantitative geometric approach, to address the Lipschitz constant estimation. By adopting this framework, we can immediately obtain several theoretical results, including the computational hardness of Lipschitz constant estimation and its approximability. We implement the algorithms induced from this quantitative geometric approach, which are based on semidefinite programming (SDP).\footnote{Our code is available at \url{https://github.com/z1w/GeoLIP}.} Our empirical evaluation demonstrates that they are more scalable and precise than existing tools on Lipschitz constant estimation for $\ell_\infty$-perturbations.
%Furthermore, the quantitative geometric perspective can also provide some insight into recent empirical observations that techniques for one norm do not usually transfer to another one.
%propose a semidefinite programming algorithm to estimate the upper bound of the Lipschitz constant. We provide tight bounds for the estimation in the case of single-hidden layer networks, using the celebrated Grothendieck inequality. This is the first known approximation guarantee to this Lipschitz constant upper bound of neural networks. 
%We also implement the algorithms induced from this quantitative geometric approach, in a tool GeoLIP. These algorithms are based on semidefinite programming (SDP). Our empirical evaluation demonstrates GeoLIP is more scalable and precise than existing tools on Lipschitz constant estimation for $\ell_\infty$-perturbations. 
Furthermore, we also show their intricate relations with other recent SDP-based techniques, both theoretically and empirically. We believe that this unified quantitative geometric perspective can bring new insights and theoretical tools to the investigation of neural-network smoothness and robustness.
\end{abstract}

\section{Introduction}\label{sec:intro}

The past decade has witnessed the unprecedented success of deep learning in many machine learning tasks~\citep{Krizhevsky_imagenetclassification,nlp_2013}. Despite the growing popularity of deep learning, researchers have also found that neural networks are very vulnerable to adversarial attacks~\citep{szegedy2014intriguing,goodfellow2015explaining,papernot2015limitations}. As a result, it is important to train neural networks that are robust against those attacks~\cite{madry2019deep}. In recent years, the deep learning community starts to focus on certifiably robust neural networks~\cite{albarghouthi2021introduction,heinNIPS,reluplex,cohen2019certified,certified_def,IUA,gloro}. 
One way to achieve certified robustness is to estimate the smoothness of neural networks, where the smoothness is measured by the Lipschitz constant of the neural network. Recent works have found that to achieve both high accuracy and low Lipschitzness, the network has to significantly increase the model capacity~\citep{bubeck2021a}. This implies that there is an intrinsic tension between the accuracy and smoothness of neural networks. 

Commonly considered adversarial attacks are the $\ell_\infty$ and $\ell_2$-perturbations in the input space. \citet{gloro,cohen2019certified} have successfully trained networks with low $\ell_2$-Lipschitz constant, and \citet{huang2021local} trained networks with low local Lipschitzness for $\ell_\infty$-perturbations.
There are few well-established techniques to train neural networks with low global Lipschitzness for $\ell_\infty$-perturbations. The techniques for the $\ell_2$-perturbation do not easily transfer to the $\ell_\infty$-case~\cite{gloro}. One critical step is to measure the Lipschitz constant more precisely and efficiently.
\citet{jordan2021exactly} showed that for ReLU networks, it is $\NP$-hard to approximate the Lipschitz constant for $\ell_\infty$-perturbations within a constant factor , and proposed an exponential-time algorithm to compute the exact Lipschitz constant. However, researchers are interested in more scalable approaches to certify and train networks. In this work, we consider the \emph{Formal Global Lipschitz} constant (FGL) (See~\cref{eq:FGL}), which is roughly the maximum of the gradient operator norm, assuming all activation patterns on hidden layers are independent and possible. FGL is an upper bound of the exact Lipschitz constant and has been used in~\citet{certified_def,lipsdp,lipopt}.

%In theory, activation functions in a neural network are correlated, and the supremum of all feasible activation patterns induces the \emph{true} global Lipschitz constant. In practice, we can treat all activations independently, and consider the maximum gradient norm over all possible activation patterns. This upper bound of Lipschitz constant, which we call the \emph{Formal Global Lipschitz} (FGL) constant, is also beneficial for certifying the robustness of neural networks~\cite{lipreg,lipopt,chen2020semialgebraic,certified_def,lipsdp,anti_lipSDP}. Throughout the paper, we mainly consider the estimation of FGL for $\ell_\infty$ and $\ell_2$-perturbations.

We address the Lipschitz constant estimation from the quantitative geometric perspective. Quantitative geometry aims to understand  geometric structures via quantitative parameters, and has connections with many mathematical disciplines, such as functional analysis and combinatorics~\citep{quant_geo}. In computer science, quantitative geometry plays a central role in understanding the approximability of discrete optimization problems. We approach those hard discrete optimization problems by considering the efficiently solvable continuous counterparts, and analyze the precision loss due to relaxation, which is often the SDP relaxation~\citep{maxcut,nest,cut-norm}. The natural SDP relaxations for the intractable problems usually induce the optimal known polynomial time algorithms~\citep{sepr}. By adopting the quantitative geometric approach, we can immediately understand the computational hardness and approximability of FGL estimations. Our algorithms on two-layer networks are the natural SDP relaxations from the quantitative geometric perspective.

\citet{lipopt} employed polynomial optimization methods on the FGL estimation for $\ell_\infty$-perturbations.
Polynomial optimization is a very general framework, and many problems can be cast in this framework~\citep{motzkin_straus_1965,maxcut}. Therefore, we argue that this is not a precise characterization of the FGL-estimation problem. On the other hand, there are also several SDP-based techniques for FGL estimations. \citet{certified_def} proposed an SDP algorithm to estimate the FGL for $\ell_\infty$-perturbations of two layer networks, and \citet{lipsdp} devised an SDP algorithm to estimate the $\ell_2$-FGL. We will demonstrate the intricate relationships between our algorithms and these existing SDPs on two-layer networks.%, and their relationships in return inspires us to design a complete framework and algorithms for FGL estimations on multi-layer networks. 

Several empirical studies have found that techniques on one $\ell_p$-perturbations often do not transfer to another ones, even though the authors claim that \emph{in theory} these techniques should transfer~\citep{lipsdp,gloro}. This in-theory claim usually comes from a \emph{qualitative} perspective. In finite-dimensional space, one can always bound one $\ell_p$-norm from another one, so techniques for one $\ell_p$-perturbations can also provide another bound for a different $\ell_p$-perturbations. However, this bound is loose and in practice not useful~\citep{lipopt}. Instead, we believe that when transferring techniques from one norm to another one, we should consider the quantitative geometric principle: we should separate the geometry-dependent component from the geometry-independent one in those techniques, and modify the geometry-dependent component accordingly for a different normed space. As we will demonstrate, our whole work is guided by this principle. We hope that our unified quantitative geometric framework can bring insights to the empirical hard-to-transfer observations, and new tools to address these issues.
%For example, a Boolean constraint can be simply expressed as $x^2=1$, and many hard combinatorial optimization problems can be formulated as a polynomial optimization problem~\citep{motzkin_straus_1965,maxcut}. Therefore, it is generally hard to solve a polynomial optimization problem~\citep{lasserre_2015}, and the polynomial optimization approaches to estimate the Lipschitz constant do not scale well to the scope of practically used networks. 

%In this work, we propose a semidefinite programming (SDP) based algorithm estimate the Lipschitz constant of $\ell_\infty$-perturbation. Our algorithm is motivated by the SDP relaxation to the matrix mixed norm problem, which has a tight theoretical bound from the celebrated Grothendieck inequality~\citep{cut-norm,GroIneq}. This line of work is pioneered by the seminal Goemans-Williamson algorithm to the MAXCUT problem~\citep{maxcut}. We will demonstrate the close relationship between the FGL estimation and the mixed norm problem. This in return provides several theoretical implications and practical results.

\paragraph{Contributions}To summarize, we have made the following contributions:
\begin{enumerate}
    \item We provide a unified theoretical framework to address FGL estimations, which immediately yields the computational hardness and SDP-based approximation algorithms on two-layer networks (\cref{sec:two-layer}). %These computational hardness and approximability results are the first for FGL estimations.
    \item We demonstrate the relations between our algorithms and other SDP-based tools, which in return inspires us to design the algorithms for multi-layer networks. This provides more insightful and compositional interpretations of existing works, and makes them easier-to-generalize (\cref{sec:dual,sec:dual_infty}). 
    \item We implement the algorithms and name the tool GeoLIP. We empirically validate our theoretical claims, and compare GeoLIP with existing methods to estimate FGL for $\ell_\infty$-perturbations. The result shows that GeoLIP provides a tighter bound (20\%-60\% improvements) than existing tools on small networks, and much better results than the naive matrix-norm-product method for deep networks, which existing tools cannot handle (\cref{sec:eval}).
\end{enumerate}

\section{Preliminaries}\label{sec:prelim}
\paragraph{Notation}
Let $[n] = \{1,\ldots,n\}$. For two functions $f$ and $g$, $f\func g(x) = f(g(x))$. A 0-1 cube is $\{0,1\}^n$, and a norm-1 cube is $\{-1,1\}^n$ for some integer $n >0$. $\R_+ = [0,\infty)$.
For any vector $v\in \R^n$, $\ACTL(v)$ is an $n\times n$ diagonal matrix, with diagonal values $v$. Let $e_n = (1,\ldots, 1)\in \R^n$ be an $n$-dimensional vector of all $1$'s; and $I_n = \ACTL(e_n)$, the identity matrix. Let $\norm{v}_p$ denote the $\ell_p$ norm of $v$. We use $q$ to denote the Hölder conjugate of $p$ as a convention, i.e., $\frac{1}{p} + \frac{1}{q}=1$. If $v$ is an operator in the $\ell_p$-space, the operator norm of $v$ is then $\norm{v}_q$. Throughout the paper, we consider the $\ell_p$-norm of the input's perturbation, and therefore, the $\ell_q$-norm of the gradient, which acts as an operator on the perturbation.
A square matrix $X\succeq 0$ means that $X$ is positive semidefinite (PSD). Let $\trace(X)$ be the trace of a square matrix $X$. If $a, b\in \R^n$, let $\inprod{a,b}$ be the inner product of $a$ and $b$. 

\paragraph{Lipschitz function}
Given two metric spaces $(X, d_X)$ and $(Y, d_Y)$, a function $f: X\rightarrow Y$ is \emph{Lipschitz} continuous if there exists $K > 0$ such that for all $x_1, x_2\in X$,
\begin{equation}\label{eq:lipDef}
    d_Y(f(x_2), f(x_1))\leq Kd_X(x_2, x_1).
\end{equation}
The smallest such $K$ satisfying~\cref{eq:lipDef}, denoted by $K_f$, is called the Lipschitz constant of $f$.
For neural networks, $X$ is in general $\mathbb{R}^m$ equipped with the $\ell_p$-norm. We will only consider the case when $Y=\R$. In actual applications such as a classification task, a neural network has multiple outputs. The prediction is the class with the maximum score. One can then use the margin between each pair of class predictions and its Lipschitz constant to certify the robustness of a given prediction~\citep{certified_def,gloro}.
From Rademacher's theorem, if $f$ is Lipschitz continuous, then $f$ is is almost everywhere differentiable, and $K_f = \sup_{x} \norm{\nabla f(x)}_q$.

\paragraph{Neural network as function}
A neural network $f: \R^m\rightarrow \R$ is characterized by: %characterized by the following computation:
\[
f_1(x)=  W^1 x+b_1;\; f_i(x)=  W^i\act(f_{i-1}(x))+b_i, i=2,\ldots,d.
\]
where $W^i\in \R^{n_{i+1}\times n_{i}}$ is the weight matrix between the layers, $n_1 = m$, $d$ is the depth of the neural network, $\act$ denotes an activation function, $b_i\in \R^{n_{i+1}}$ is the bias term, and $f = f_d\func\cdots\func f_1$. Because we only consider the $\R$ as the codomain of $f$, $W^d = u \in \R^{1\times n_{d}}$ is a vector. From chain rule, the gradient of this function is 
\begin{equation}\label{eq:grad}
    \nabla f(x) = (W^1)^T[\ACTL(\act'(f_{1}(x))) (W^2)^T \cdots \ACTL(\act'(f_{d-1}(x))) (W^d)^T].
\end{equation}
Common activation functions, including $\relu$~\citep{relu}, sigmoid functions, and ELU~\citep{Clevert2016FastAA} are almost everywhere differentiable. As a result, we are interested in the supremum operator norm of~\cref{eq:grad}.

However, checking all possible inputs $x$ is infeasible, and common activation functions have bounded derivative values, say $[a,b]$. We are then interested in the following value instead:
\begin{equation}\label{eq:FGL}
    \max_{v^i\in[a,b]^{n_i}} \norm{(W^1)^T \cdot \ACTL(v^2) \cdot\cdots\cdot \ACTL(v^d) (W^{d})^T}_q,
\end{equation}
where $n_i$ is the dimension of each $\ACTL(v^i)$. We call this value the \emph{formal} global Lipschitz constant (FGL) because we treat all activation functions independent but in reality not all activation patterns are feasible. Therefore, this is an upper bound of the \emph{true} global Lipschitz constant of the neural network. However, it is the value studied in most global Lipschitzness literature~\citep{lipreg,lipsdp,lipopt}, and also turns out useful in certifying the robustness of neural networks~\citep{certified_def,gloro,anti_lipSDP}. We use $\ell_p$-FGL to denote the FGL for $\ell_p$-perturbations.

In this paper, we focus on the $\ell_p$-perturbation on the input, where $p=\infty$ or $p=2$. Because $q$ is the Hölder conjugate of $p$, we are interested in the value of~\cref{eq:FGL}, when $q = 1$ (for $p=\infty$), and $q = 2$ (for $p=2$).
Notice that in the $\relu$-network case, $[a, b]$ is $[0,1]$. We will use $\relu$-networks as the illustration for the rest of the paper because of the popularity of $\relu$ in practice and the easy presentation of the 0-1-cube. However, the algorithms presented in this work can be adapted with minor adjustments to other common activation functions. %For the maximization problem considered in the work, we say an algorithm $\mathcal{A}$ has an $\alpha$-approximation ratio, if for all instances
\begin{remark}
FGL considers all possible activation patterns on the hidden layers, while some of the activation patterns might be unachievable in reality. Therefore, FGL is an upper bound of the true Lipschitz constant. Notice that the activation pattern induced from an input is also decided by the bias term. Therefore, to find the true Lipschitz constant, one has to incorporate the information from the bias term.
\end{remark}
\section{Two-layer neural networks}\label{sec:two-layer}
In this section, we consider the two-layer neural network case. We reduce the FGL estimation to the matrix mixed-norm problem. This immediately yields the computational complexity and approximation algorithms for FGL estimations. In~\cref{sec:analysis}, we show that we can consider $\{0,1\}$ instead of $[0,1]$ in~\cref{eq:FGL} for two-layer networks.

\paragraph{Problem description}%The two-layer neural networks can be defined as $f(x)= u  \ACTL(y) Wx$, %in the following way:
%\[
%    f(x)= u  \ACTL(y) Wx,
%\]
%where $W\in \R^{n\times m}$, $\ACTL(y)\in \R^{n\times n}$, and $u\in \R^{1\times n}$. 
For a two-layer network where $W^1 = W\in \R^{n\times m}$ and $W^d = u\in \R^{1\times n}$,
its FGL (as in~\cref{eq:FGL}) is
$\max_{y\in\{0, 1\}^n} \norm{W^T \ACTL(y) u^T}_q$, where we use $y$ to denote $v^1$ in this case.
If we expand the matrix multiplication, it is easy to check that this equals to $\max_{y\in\{0, 1\}^n}\norm{ W^T \ACTL(u) y}_q$. %the following equation:
%\[
%    \max_{y\in\{0, 1\}^n}\norm{ W^T \ACTL(u) y}_q.
%\]
Let $A = W^T \ACTL(u)$, then the $\ell_p$-FGL is
\begin{equation}\label{eq:two_layer_obj}
    \max_{y\in\{0, 1\}^n} \norm{A y}_q.
\end{equation}
\subsection{$\ell_\infty$-FGL estimation}%\label{sec:rescale}
We consider a natural SDP relaxation to~\cref{eq:two_layer_obj} when $q=1$, and analyze the result using the celebrated \emph{Grothendieck Inequality}, which is a fundamental tool in functional analysis.

\paragraph{Mixed-norm problem} The $\infty\rightarrow 1$ mixed-norm of a matrix is defined as 
\[
\norm{A}_{\infty\rightarrow 1} = \max_{\norm{x}_\infty=1}\norm{Ax}_1.
\]
The mixed-norm problem appears similar to \cref{eq:two_layer_obj} when $q=1$, except for that instead of a norm-$1$-cube, the cube in \cref{eq:two_layer_obj} is a 0-1-cube. 
\citet{cut-norm} showed that it is $\NP$-hard, specifically $\MAXSNP$-hard, to compute the $\infty\rightarrow 1$ mixed-norm of a matrix $A$, via a reduction to the graph $\MAXCUT$ problem. 
%
%\citet{cut-norm} showed that the mixed norm of a matrix is $\NP$-hard, specifically $\MAXSNP$ hard In the following, we will demonstrate some cube rescaling techniques between the 0-1-cube and the norm-$1$ cube. They will provide several theoretical and practical results for the FGL estimation problem.
%
%\paragraph{SDP for the Mixed-Norm Problem}
Moreover, \citet{cut-norm} constructed a natural SDP relaxation for the mixed-norm problem:
\begin{align}\label{eq:sdp_mnp1}
\begin{split}
    \max \trace(BX)\\ 
    s.t.\;  X\succeq 0, X_{ii}=1, & i\in [n+m],
\end{split}
\end{align}
where $A$ is a submatrix of $B$. We provide the detailed derivation of this relaxation in~\cref{sec:mixed-norm}.
In fact, this relaxation admits a constant approximation factor. \citet{GroIneq} developed the local theory of Banach spaces, and showed that there exists an absolute value $K_G$ such that
\begin{theorem}\label{thm:groth-local}
For any $m, n\geq 1$, $A\in \R^{n\times m}$, and any Hilbert space $H$, the following holds:
\[
    \max_{u_i, v_j\in B(H)}\sum_{i,j} A_{ij}\inprod{u_i, v_j}_H \leq K_G\norm{A}_{\infty\rightarrow 1},
\]
\end{theorem}
where $B(H)$ denotes the unit ball of the Hilbert space.

The precise value of $K_G$ is still an outstanding open problem, and it is known that $K_G< 1.783$~\citep{KRIVINE197916,GroCKri}. 
The approximation factor of the SDP relaxation in~\cref{eq:sdp_mnp1} is $K_G$.
%\paragraph{Complexity of FGL Estimation} 
 %The cut-norm of a matrix is essentially the same as defined in~\cref{eq:nm1} except for that the constraint is a 0-1-cube. As a result, one can view the FGL estimation as a mixture of the cut-norm and the mixed-norm problems.
Similar to the mixed-norm problem, we show that the $\ell_\infty$-FGL estimation is $\MAXSNP$-hard and provide an SDP relaxation, which also admits the $K_G$-approximation ratio. We provide a detailed explanation on why $K_G$ is the approximation ratio and how we can view the SDP relaxation as a geometric transformation in~\cref{sec:grothSDP}.
\begin{theorem}\label{thm:hardness}
$\ell_\infty$-FGL estimation is $\MAXSNP$-hard.
\end{theorem}

%Let $A\in \R^{m\times n}$, and $e = (1,\ldots, 1)\in \R^n$. The mixed norm of $A$ can be obtained via the following integer program:
%\[
%    \max_{x\in \{-1, 1\}^n, y\in \{-1, 1\}^m} \inprod{Ax, y}
%\]
%Let $t = (x+e)/2$, then $x = 2t-e$. The mixed norm becomes
%\begin{align}\label{eq:mixed-norm-orig}
%\begin{split}
%    &\max_{x\in \{-1, 1\}^n, y\in \{-1, 1\}^m} \inprod{Ax, y} \\= &\max_{t\in \{0, 1\}^n, y\in \{-1, 1\}^m} y^TA(2t-e).
%\end{split}
%\end{align}
%Now introduce a new variable $\tau\in \{0,1\}$, and consider 
%\begin{equation}\label{eq:mixed-norm-new}
%\max_{(t,\tau)\in \{0, 1\}^{n+1}, y\in \{-1, 1\}^m} y^TA(2t-\tau e).
%\end{equation}

%Let $OPT_1$ be the optimal value of~\cref{eq:mixed-norm-orig} and $OPT_2$ be the optimal value of~\cref{eq:mixed-norm-new}. We will show that $OPT_1 = OPT_2$. It is clear that $OPT_2\geq OPT_1$. If $(\hat{t}, 0, \hat{y})$ is the optimal solution to~\cref{eq:mixed-norm-new}, then  

\paragraph{From 0-1 cube to norm-1 cube} If we can transform the 0-1 cube in~\cref{eq:two_layer_obj} to a norm-1 cube, and formulate an equivalent optimization problem, then one can apply the SDP program in~\cref{eq:sdp_mnp1} to compute an upper bound of the FGL. Indeed, we provide a cube rescaling technique, and it allows us to construct the SDP for the $\ell_\infty$-FGL estimation. We provide the full detail of this technique in~\cref{sec:rescale}, and the result SDP for the $\ell_\infty$-FGL estimation is
\begin{align}\label{eq:sdp_inf}
\begin{split}
    \max \;\frac{1}{2}&\trace(BX)\\ 
    s.t.\;  X \succeq 0, X_{ii}=1&, i\in [n+m+1],
\end{split}
\end{align}
where $B$ is a $(n+1+m)\times (n+1+m)$ matrix, and 
$B = \begin{pmatrix}
0 & 0 & 0\\
A & Ae_n & 0
\end{pmatrix}$. As a result, we have: % the following theorem:
\begin{theorem}\label{thm:infty-FGL}
There exists a polynomial-time approximation algorithm to estimate the $\ell_\infty$-FGL of two-layer neural networks, moreover, the approximation ratio is $K_G$.
\end{theorem}

\subsection{$\ell_2$-FGL estimation}\label{sec:nat_l2}
\citet{lipreg} showed that the $\ell_2$-FGL estimation is $\NP$-hard. If $q=2$ in~\cref{eq:two_layer_obj}, the objective becomes similar to the $\infty\rightarrow 2$ mixed-norm problem. %In fact, Grothendieck proved that $K_G\geq \frac{\pi}{2}$ when $A$ is PSD and $u = v$~\citep{GroIneq,Lindenstrauss1968,cut-norm}. 
This is a quadratic optimization problem with a PSD weight matrix over a cube, and can be viewed as a generalization of the graph Max-Cut problem. In the quadratic optimization formulation of Max-Cut, the weight matrix is the Laplacian of the graph, a special PSD matrix~\citep{maxcut}. \citet{nest} generalized Goemans-Williamson's technique and analyzed the case when the weight matrix is PSD, showing that the natural SDP relaxation in this case has a $\frac{\pi}{2}$-approximation ratio. This provides a $\sqrt{\frac{\pi}{2}}$-approximation algorithm for the $\ell_2$-FGL estimation. The approximation ratio comes from a similar inequality to the one in~\cref{thm:groth-local}, known as the \emph{Little Grothendieck Inequality}.
The SDP for $\ell_2$-FGL estimation is:
\begin{align}\label{eq:l2-SDP}
\begin{split}
   \max \frac{1}{2}&\sqrt{\trace(\begin{pmatrix}
A^TA & A^TAe_n\\
e_n^TA^TA & e_n^TA^TAe_n
\end{pmatrix} X)}\\
    s.t.\; & X \succeq 0, X_{ii}=1, i\in [n+1].
\end{split}
\end{align}
%where $\hat{M} = \begin{pmatrix}
%A^TA & A^TAe_n\\
%e_n^TA^TA & e_n^TA^TAe_n
%\end{pmatrix}$. 
The full derivation is provided in~\cref{sec:l2fgl}, and we have the following theorem:

\begin{theorem}\label{thm:2-FGL}
There exists a polynomial-time approximation algorithm to estimate the $\ell_2$-FGL of two-layer neural networks with an approximation factor $\sqrt{\frac{\pi}{2}}$.
\end{theorem}

\begin{remark}\label{rm:mixed-norm}
As we have discussed, for two-layer networks, the $\ell_p$-FGL estimation is essentially the $\infty\rightarrow q$ mixed-norm problem. Indeed the mixed-norm problem is an outstanding topic in theoretical computer science. As discussed in~\citet{pqnorm}, the $\infty\rightarrow q$ mixed norm problem has constant approximation algorithms if $q\leq 2$, and is hard to approximate within almost polynomial factors when $q>2$. Because when $q>2$, its Hölder conjugate $p<2$. This implies that for two-layer networks, the FGL estimation can be much harder for $\ell_p$-perturbations when $p<2$. 

\citet{opt_inf2} showed that it is $\NP$-hard to approximate the $\infty\rightarrow 2$ mixed-norm problem better than $\sqrt{\frac{\pi}{2}}$. \citet{UGC_inf1_opt} proved that assuming the unique games conjecture~\citep{UGC}, it is $\NP$-hard to approximate the $\infty\rightarrow 1$ mixed-norm problem better than $K_G$. These optimal approximation ratios match our SDP relaxations for FGL estimations accordingly.
\end{remark}

 %\cref{tb:existing-SDP} presents existing SDP works and their scope of application.

\section{Relations to existing SDP works}\label{sec:dual}
Before introducing our approach for multi-layer networks, we first examine some existing SDP works on FGL estimations, and discuss their relationships with our natural SDP relaxations in~\cref{sec:two-layer}.
\paragraph{$\ell_\infty$-FGL estimation}
\citet{certified_def} formulated an SDP that only works for two-layer networks. %If we compare their formulation with ours, theirs is essentially
%\begin{align*}
%\begin{split}
%    \max &\frac{1}{4}\trace\Big(\begin{pmatrix}
%0 & 0 & A^T\\
%0 & 0 & e_n^T A^T\\
%A & Ae_n & 0
%\end{pmatrix}\cdot X\Big)\\ 
%    s.t.\;  &X \succeq 0, X_{ii}=1, i\in[n+m+1].
%\end{split}
%\end{align*}
%where $C$ is a $(m+n+1)\times (m+n+1)$ matrix. %, and 
%$C = \begin{pmatrix}
%0 & 0 & A^T\\
%0 & 0 & e_n^T A^T\\
%A & Ae_n & 0
%\end{pmatrix}$. 
%This 
Theirs is essentially the same as ours in~\cref{eq:sdp_inf} (See the detailed comparison in~\cref{sec:compar}).
However, we provide a rigorous derivation and simpler formulation, and also a sound theoretical analysis of the bound, which illustrate more insights to this problem. \citet{certified_def} treated the SDP relaxation as a heuristic to a hard quadratic
programming problem. We prove that this relaxation is not only a heuristic, but in fact induces an approximation algorithm with a tight bound.

%\begin{wraptable}{r}{7.9cm}
%\vskip -0.2in
%\caption{Existing SDP relaxations for %$\ell_p$-FGL.}\label{tb:existing-SDP}
%\vskip -0.1in
%\begin{tabular}{ccc}\\\toprule  
%& $\ell_2$-FGL & $\ell_\infty$-FGL \\
%\midrule
%Two-Layer   & LipSDP& \citet{certified_def}\\
%Multi-Layer  &  LipSDP & Not Available\\
%\bottomrule
%\end{tabular}
%\vskip -0.1in
%\end{wraptable} 
%\begin{table}[t]
%\caption{Existing SDP relaxations for $\ell_p$-FGL.}\label{tb:existing-SDP}
%\label{sample-table}
%\vskip -0.5in
%\begin{center}
%\begin{small}
%\begin{sc}
%\begin{tabular}{lccccr}
%\toprule
%& $\ell_2$-FGL & $\ell_\infty$-FGL \\
%\midrule
%Two-Layer Network   & LipSDP-Neuron & \citet{certified_def}\\
%Multi-Layer Network   &  LipSDP-Neuron & Not Available\\
%\bottomrule
%\end{tabular}
%\end{sc}
%\end{small}
%\end{center}
%\vskip -0.2in
%\end{table}

\paragraph{$\ell_2$-FGL estimation}%\label{sec:dual}
%\paragraph{LipSDP for $\ell_2$-FGL estimation}
\citet{lipsdp} proposed LipSDP, another SDP-based algorithm for the $\ell_2$-FGL estimation problem. \citet{lipsdp} provided several variants of LipSDP to balance the precision and scalability. \citet{anti_lipSDP} demonstrated that the most precise version of LipSDP, \emph{LipSDP-Network}, fails to produce an upper bound for $\ell_2$-FGL. %Here we will study the less precise version, LiPSDP-Neuron, and 
In this paper, all the references of LipSDP are to \emph{LipSDP-Neuron}, the less precise version.
Surprisingly, even though the approach in LipSDP appears quite different from~\cref{eq:l2-SDP}, % the mixed-norm formulation, %the empirical evaluation shows that the results are exactly the same. 
we show that LipSDP is dual of~\cref{eq:l2-SDP} to estimate the $\ell_2$-FGL on two-layer networks.
LipSDP for two-layer networks is: %for the network is
\[
\min_{\obju,\dualv}\Big\{\sqrt{\obju}:\begin{pmatrix}
-2ab W^TWT-\obju I_m & (a+b)W^T T \\
(a+b) T W & -2T+u^Tu
\end{pmatrix}\preceq 0, \dualv_i\geq 0\Big\},
\]
where $T = \ACTL(\dualv)$ for $\dualv\in \R_+^n$; $a$ and $b$ are the lower and upper bounds of the activation's derivative. %This SDP program appears very different from~\cref{eq:l2-SDP}, but we will show they are equivalent.

%Let us assume that $\act'(x)\in [a,b]$. Because the gradient is $W^T\ACTL(Y) u = W^T\ACTL(u) Y$, let $M =W^T\ACTL(u)$, and the square of the $\ell_2$-norm is 

%\paragraph{A New Optimization Program} 
We will construct a new quadratic program, which we show is equivalent to~\cref{eq:two_layer_obj} when $q=2$, and LipSDP is its dual SDP relaxation.

Let the input of the $i$-th activation node on $\ACTL(y)$ be $y_i$, and $w_i$ be the row vector of $W$. Hence, $y_i = w_ix$. Let $\Delta x\in \R^m$ be a perturbation on $x$, so $\Delta y_i = w_i \Delta x$. Let
$\Delta \act(y)\in \R^n$ denote the induced perturbation on $\ACTL(y)$.
The \textbf{\emph{constraint}} from the activation function is $\frac{\Delta \act(y)_i}{\Delta y_i}\in [a,b]$, in other words, $\frac{\Delta \act(y)_i}{w_i \Delta x}\in [a,b]$. 
One can write the range constraint as 
\[
    (\Delta \act(y)_i - a\cdot w_i \Delta x)(\Delta \act(y)_i - b\cdot w_i \Delta x) \leq 0.
\]
This can be written in the quadratic form:
\begin{equation}\label{eq:lipsdp_c}
\begin{pmatrix}
w_i \Delta x\\
\Delta \act(y)_i
\end{pmatrix}^T \begin{pmatrix}
-2ab & a+b \\
a+b & -2 
\end{pmatrix}\begin{pmatrix}
w_i \Delta x\\
\Delta \act(y)_i
\end{pmatrix}\geq 0,\;\;\;\;\forall i \in [n].
\end{equation}
Since for the two layer network, $f(x) = u \act(y)$, then $\Delta f(x) = u\Delta \act(y)$. The \textbf{\emph{objective}} for $\ell_2$-FGL estimation is $\max_{\Delta x,\Delta \act(y)} \sqrt{\frac{(u\Delta \act(y))^2}{(\Delta x)^2}}$. The equivalence between this program and~\cref{eq:two_layer_obj} when $q=2$ is presented in~\cref{sec:equiv}.

\begin{remark}
Another interpretation for the quadratic program is that we want to quantify how the output changes given a data-independent input change, i.e., $\Delta x$. In other words, we want to analyze the effect of $\Delta x$ propagating from the input to the output, with symbolic values rather than actual inputs. The idea is similar to symbolic execution from program analysis~\citep{SurveySymExec}.
\end{remark}

\paragraph{Duality to LipSDP} 
Now we will show that LipSDP is the dual SDP to the program formulated above. The dual SDP derivation is of similar form in~\citet[Ch.4.3.1]{modern_co}. Let us introduce a variable $\obju$ such that $\obju - \frac{(u\Delta \act(y))^2}{(\Delta x)^2}\geq 0$.
In other words,
\begin{equation}\label{eq:newopt_obj}
\obju (\Delta x)^2 - (u\Delta \act(y))^2\geq 0.
\end{equation}
For each constraint in~\cref{eq:lipsdp_c}, let us introduce a dual variable $\dualv_i\geq 0$. Multiply each constraint with $\dualv_i$, then
$
\begin{pmatrix}
\Delta x\\
\Delta \act(y)_i
\end{pmatrix}^T \begin{pmatrix}
-2ab \dualv_i w_i^Tw_i & (a+b)\dualv_iw_i^T  \\
(a+b)\dualv_i w_i & -2\dualv_i 
\end{pmatrix}\begin{pmatrix}
\Delta x\\
\Delta \act(y)_i
\end{pmatrix}\geq 0,\;\;\;\;\forall i \in [n].
$

Sum all of them, then we have 
$\begin{pmatrix}
\Delta x\\
\Delta \act(y)
\end{pmatrix}^T \begin{pmatrix}
-2ab W^TW T & (a+b)W^T T \\
(a+b) TW & -2T
\end{pmatrix}\begin{pmatrix}
\Delta x\\
\Delta \act(y)
\end{pmatrix}\geq 0$,
where $T = \ACTL(\dualv)$ is the $n\times n$ diagonal matrix of dual variables $\dualv_1,\ldots,\dualv_n$.

\cref{eq:newopt_obj} can be rewritten as:
$
\begin{pmatrix}
\Delta x\\
\Delta \act(y)
\end{pmatrix}^T \begin{pmatrix}
\obju I_{m} & 0 \\
0 & -u^Tu
\end{pmatrix}\begin{pmatrix}
\Delta x\\
\Delta \act(y)
\end{pmatrix}\geq 0.
$
As a result, the dual program for the new optimization program is
\[
\min_{\obju,\dualv}\Big\{\sqrt{\obju}:\begin{pmatrix}
-2ab W^TWT-\obju I_m & (a+b)W^T T \\
(a+b) T W & -2T+u^Tu
\end{pmatrix}\preceq 0, \dualv_i\geq 0\Big\}.
\]
%which is exactly LipSDP. 
%
\begin{remark}
In~\cref{rm:mixed-norm}, we mention that $\sqrt{\frac{\pi}{2}}$ is the optimal approximation ratio for the $\infty\rightarrow 2$ mixed-norm problem, which matches the approximation ratio in~\cref{thm:2-FGL}. Hence, improving the natural SDP relaxation in~\cref{eq:l2-SDP} can be very hard. The duality provides another evidence of LipSDP-Neuron's correctness, and hints that LipSDP-Network, the improved variant, may be wrong.
\end{remark}
\section{$\ell_\infty$-FGL estimation for multi-layer networks}\label{sec:dual_infty}
For a multi-layer neural network, the formal gradient becomes a high-degree polynomial, and its $\ell_q$-norm estimation becomes a high-degree polynomial optimization problem over a cube, which is in general a hard problem~\citep{lasserre_2015}. We provide a discussion of the polynomial optimization approach of FGL estimation in~\cref{sec:poly}. %Here, we provide an SDP dual program of $\ell_\infty$-FGL estimation inspired by the dual optimization approach in~\cref{sec:dual}.  %However, \citet{lipsdp} devised a compact matrix representation of multi-layer neural networks so that there is no need to handle high-degree polynomials. We will show that this compact matrix representation of neural network has a natural dual program interpretation, which allows us to transfer the technique to the $\ell_\infty$-FGL estimation.
%
%We will first give an SDP dual program of $\ell_\infty$-FGL estimation on two-layer networks. The multi-layer case follows easily. In the end, we will briefly discuss the gradient-polynomial approach to multi-layer network FGL estimations.
%
%\subsection{Dual for $\ell_\infty$-FGL estimation}\label{sec:dual_inf}
%We have shown the duality between LipSDP and the natural SDP relaxation on $\ell_2$-FGL estimation for two layer networks. % It would be interesting to formulate a similar treatment on $\ell_1$-FGL, so that we can bypass both the relaxations proposed above.
%
Here we provide an SDP dual program of the $\ell_\infty$-FGL estimation inspired by the dual SDP approach in~\cref{sec:dual}. The difference is that now we consider $\ell_\infty$-perturbations to the input space instead of $\ell_2$. Hence, the objective becomes
\[
    \max_{\Delta x, \Delta \act(y)}\frac{|u\Delta \act(y)|}{\norm{\Delta x}_\infty}.
\]

%Because $u\in \R^{1\times n}$, $u\Delta \act(y)$ is a scalar, and so $|u\Delta \act(y)| = \sqrt{(u\Delta \act(y))^2}$
If we add an extra constraint $\norm{\Delta x}_\infty = 1$, the above objective becomes
\begin{equation}\label{eq:new_dual_obj}
    \max_{\Delta x, \Delta \act(y)} \frac{1}{2}(u\Delta \act(y)+ u\Delta \act(y)).
\end{equation}
%We have $u\Delta \act(y)+ u\Delta \act(y)$, because $u\Delta \act(y)$ is a scalar, and we need a symmetric matrix representation in the SDP program.
The constraints are
\begin{equation*}
\begin{pmatrix}
w_i \Delta x\\
\Delta \act(y)_i
\end{pmatrix}^T \begin{pmatrix}
2ab & -(a+b) \\
-(a+b) & 2 
\end{pmatrix}\begin{pmatrix}
w_i \Delta x\\
\Delta \act(y)_i
\end{pmatrix}\leq 0,\; \Delta(x)_j^2 \leq 1,\;\forall i \in [n],j\in [m].
\end{equation*}

We can write $\Delta(x)_j^2 \leq 1$ as
$\begin{pmatrix}
1\\
\Delta x_j
\end{pmatrix}^T \begin{pmatrix}
1 & 0 \\
0 & -1 
\end{pmatrix}\begin{pmatrix}
1\\
\Delta x_j
\end{pmatrix}\geq 0$.

Now let us introduce $n+m$ non-negative dual variables $(\tau, \dualv)$, where $\tau\in \R^m_+$ and $\dualv\in \R^n_+$. If we multiply each dual variable with the constraint and add all the constraints together, we will have
\[
\begin{pmatrix}
1\\
\Delta x\\
\Delta \act(y)
\end{pmatrix}^T \begin{pmatrix}
\sum_{j=1}^m \tau_j & 0 & 0\\
0 & -2ab W^TW T_2 - T_1 & (a+b)W^T T_2\\
0 & (a+b) T_2 W & -2T_2
\end{pmatrix}\begin{pmatrix}
1\\
\Delta x\\
\Delta \act(y)
\end{pmatrix}\geq 0,
\]
where $T_1 = \ACTL(\tau)$ and $T_2 = \ACTL(\dualv)$.
As a result, we can incorporate the objective~\cref{eq:new_dual_obj} and  obtain the dual SDP for the $\ell_\infty$-FGL estimation:
\begin{equation}\label{eq:dual_inf}
\min_{\obju,\dualv, \tau}\Big\{\frac{\obju}{2}: \begin{pmatrix}
\sum_{j=1}^m \tau_j-\obju& 0 & u\\
0 & -2ab W^TW T_2 - T_1 & (a+b)W^T T_2\\
u^T & (a+b) T_2 W & -2T_2
\end{pmatrix}\preceq 0, \dualv_i, \tau_j\geq 0\Big\}.   
\end{equation}

\begin{remark}
The SDP programs in~\cref{sec:two-layer} are strictly feasible because the identity matrix is a positive definite solution. Hence, Slater's condition is satisfied and strong duality holds. 
\end{remark}
\paragraph{Multi-layer extension} 
We can simply extend the dual program to multiple-layer networks. We first vectorize all the units in the input layer and hidden layers, and then constrain them using layer-wise inequalities to formulate an optimization problem.
Let us consider a general $d$-layer multi-layer network, where $W^i\in \R^{n_{i+1}\times n_{i}}$ for $i\in [d-1]$, and $W^{d} = u\in \R^{1\times n_d}$. 
Let $\Delta x$ denote the perturbation on the input layer, $\Delta z^i$ be the perturbation on the $i$-th hidden layer, and $w^i_j$ be the $j$-th row vector of $W^i$. The only difference between two layer networks and multi-layer networks is that we have the additional constraints:
%\begin{equation}\label{eq:layer1}
%\begin{pmatrix}
%\Delta x\\
%\Delta z^1_k
%\end{pmatrix}^T \begin{pmatrix}
%-2ab (w^1_k)^Tw^1_k & (a+b)(w^1_k)^T \\
%(a+b) w^1_k & -2
%\end{pmatrix}\begin{pmatrix}
%\Delta x\\
%\Delta z^1_k
%\end{pmatrix}\geq 0,
%\end{equation}
\[\begin{pmatrix}
\Delta z^i\\
\Delta z^{i+1}_j
\end{pmatrix}^T \begin{pmatrix}
-2ab (w^{i+1}_j)^Tw^{i+1}_j & (a+b)(w^{i+1}_j)^T \\
(a+b) w^{i+1}_j & -2
\end{pmatrix}\begin{pmatrix}
\Delta z^i\\
\Delta z^{i+1}_j
\end{pmatrix}\geq 0.  
\]
%and also the box constraint on $\Delta x$:
%$\begin{pmatrix}
%1\\
%\Delta x_l
%\end{pmatrix}^T \begin{pmatrix}
%1 & 0 \\
%0 & -1 
%\end{pmatrix}\begin{pmatrix}
%1\\
%\Delta x_l
%\end{pmatrix}\geq 0$. 

%The objective is: $
%    \max \frac{1}{2}(u\cdot \Delta z^{d-1} + u\cdot \Delta z^{d-1} )$.
Let $\dualvec_i\in \R^{n_{i}}_+$ and $T_i = \ACTL(\dualvec_i)$ for $i\in[d]$. Following the similar SDP dual approach, we can add all the constraints together and formulate the following SDP program:
\begin{equation}\label{eq:geolip}
   \min_{\obju, \dualvec_{i}}\Big\{\frac{\obju}{2}: (L+N)\preceq 0, i\in[d] \Big\},  
\end{equation}
where 
\begin{equation}\label{eq:contraintL}
L = \begin{pmatrix}
0& 0 & 0 & \dots&0 & 0\\
0 & -2ab (W^1)^T W^1 T_2 & (a+b)(W^1)^T T_2 &\dots&0 & 0\\
0&(a+b) T_2 W^1 & -2T_2-2ab (W^2)^TW^2T_3&\dots&0 & 0\\
\vdots&\vdots & \vdots&\ddots & \vdots & \vdots\\
%0&0 & 0&0 & -2abW_d^TW_d T_d - 2T_{d-1} & (a+b)W_d^T T_d\\
0&0 &0&\dots & (a+b)T_dW^{d-1} & -2T_d\\
\end{pmatrix},
\end{equation}
\[
N = \begin{pmatrix}
\sum_{k=1}^{n_1} \dualvec_{1k}-\obju& 0 & \dots & u\\
0 & -T_1 & \dots &0\\
\vdots&\vdots & \ddots&\vdots\\
u^T&0 & \cdots&0
\end{pmatrix}.
\]

\begin{remark}
If we expand the matrix inequality derived from the compact neural-network representation in~\citet[Theorem 2]{lipsdp}, we will have exactly the same matrix for network constraints as $L$ (\cref{eq:contraintL}) in the dual program formulation. In other words, we provide a compositional optimization interpretation to the compact neural-network representation in LipSDP. With this interpretation, one can extend the SDP to beyond feed-forward structures, such as skip connections~\citep{resnet}. Notice that if we apply the similar reasoning to the multi-layer network $\ell_2$-FGL estimation, we will obtain LipSDP-Neuron.   %This dual optimization interpretation can benefit future research in the study of neural networks.
\end{remark}

\section{Evaluation and discussion}\label{sec:eval}
The primary goal of our work is to provide a theoretical framework, and also algorithms for $\ell_\infty$-FGL estimations on practically used networks. The $\ell_2$-FGL can be computed using LipSDP. %~\citep{lipsdp}. 
We have implemented the algorithms using MATLAB~\citep{MATLAB}, the CVX toolbox~\citep{cvx} and MOSEK solver~\citep{mosek}, and name the tool GeoLIP. To validate our theory and the applicability of our algorithms, we want to empirically answer the following research questions:
\begin{tcolorbox}
\begin{enumerate}[start=1,label={\bfseries RQ\arabic*:}]
\item Is GeoLIP better than existing methods in terms of precision and scalability?
\item Are the dual SDP programs devised throughout the paper valid?
\end{enumerate}
\end{tcolorbox}
As we shall see, GeoLIP is indeed better than existing methods in terms of precision and scalability; and the dual SDP programs produce the same values as their natural-SDP-relaxation counterparts.

%We will conduct several experiments to address these research questions.
\subsection{Experimental design}\label{sec:ex-design}
To answer \textbf{RQ1}, we will run GeoLIP and existing tools that measure the $\ell_\infty$-FGL on various feed-forward neural networks trained with the MNIST dataset~\citep{mnist}. %The tools and network configurations are specified later. 
We will record the computed $\ell_\infty$-FGL to compare the precision, and the computation time to compare the scalability. 

To answer \textbf{RQ2}, we will run the natural SDP relaxations for $\ell_p$-FGL estimations proposed in~\cref{sec:two-layer}, LipSDP for $\ell_2$-FGL estimation, and the dual program~\cref{eq:dual_inf} for $\ell_\infty$-FGL on two-layer neural networks, and compare their computed FGLs.

\begin{table}[t]
\caption{$\ell_\infty$-FGL estimation of various methods: DGeoLIP and NGeoLIP induce the same values on two layer networks. DGeoLIP always produces tighter estimations than LiPopt and MP do.}\label{tb:selected}
%\label{sample-table}
\begin{center}
\begin{small}
%\begin{sc}
\begin{tabular}{ccccccc}
\toprule
Network&DGeoLIP & NGeoLIP & LiPopt & MP & Sample & BruF \\
\midrule
%2-layer/8 units    & 142.19& 142.19& 180.38& 411.90& 134.76 & 134.76\\
2-layer/16 units   &185.18& 185.18& 259.44 & 578.54 & 175.24 & 175.24\\
%64    & 287.60&287.60& 510.00& 1207.70 & 253.89 & N/A\\
%128    & 346.27&346.27& 780.46& 2004.34 & 266.22 & N/A\\
2-layer/256 units    & 425.04&425.04& 1011.65& 2697.38 & 306.98 & N/A\\
8-layer/64 units per layer &8327.2 &-----& N/A& $8.237*10^7$ & 1130.6 & N/A\\
%512    & 537.80& 1452.41& 4755.23 &  376.06\\
\bottomrule
\end{tabular}
%\end{sc}
\end{small}
\end{center}
\end{table}

\begin{table}
\centering
\caption{Running time (in seconds) of various tools on $\ell_\infty$-FGL estimation: DGeoLIP and NGeoLIP are faster than LiPopt. Notice that the running time is implementation and solver-dependent.}\label{tb:selected-time}
\begin{small}
\begin{tabular}{ccccc}\\\toprule  
Network& DGeoLIP & NGeoLIP & LiPopt & BruF \\\midrule
2-layer/16 units    & 28.1& 22.3& 1572 & 4.8\\  %\midrule
2-layer/256 units    & 976.0& 70.9& 2690 & N/A\\  %\midrule
8-layer/64 units    & 329.5& -----& N/A & N/A\\  \bottomrule
\end{tabular}
\end{small}
\end{table}
\paragraph{Measurements}
Our main baseline tool is \textbf{\emph{LiPopt}}~\citep{lipopt}, which is an $\ell_\infty$-FGL estimation tool.\footnote{Another method was proposed by~\citet{chen2020semialgebraic}, however, the code is not available and we are not able to compare it with GeoLIP.} Notice that LiPopt is based on the Python Gurobi solver~\citep{gurobi}, while we use the MATLAB CVX and MOSEK solver.
LiPopt relies on a linear programming (LP) hierarchy for the polynomial optimization problem. We use LiPopt-k to denote the $k$-th degree of the LP hierarchy.
\textbf{\emph{BruF}} stands for an brute-force exhaustive enumeration of all possible activation patterns. This is the ground truth for FGL estimations. However, this is an exponential-time search, so we can only run it on networks with a few hidden units.
\textbf{\emph{Sample}} means that we randomly sample $200,000$ points in the input space and compute the gradient norm at those points. Notice that this is a lower bound of the true Lipschitz constant, and thus a lower bound of the FGL.
\textbf{\emph{MP}} stands for the weight-matrix-norm-product method. This is a naive upper bound of FGL. We use \textbf{\emph{NGeoLIP}} to denote the natural SDP relaxations devised in~\cref{sec:two-layer}, and \textbf{\emph{DGeoLIP}} to denote the dual SDP~\cref{eq:geolip} for $\ell_\infty$-FGL estimation. Notice that NGeoLIP only applies to two-layer networks.

%\paragraph{SETUP} 
%All the experiments are run on a workstation with forty-eight Intel\textsuperscript{\textregistered} Xeon\textsuperscript{\textregistered} Silver 4214 CPUs running at 2.20GHz, and 258 GB of memory. 
We use ``---'' in the result tables to denote that the experimental setting is not in the scope of the tool's application, and  ``N/A'' to denote the computation takes too much time ($>20$ hours).

%\begin{wraptable}{r}{7.1cm}
%\vskip -0.15in
%\caption{Running time (in seconds) of various tools on $\ell_\infty$-FGL estimation: DGeoLIP and NGeoLIP are faster than LiPopt.}\label{tb:selected-time}
%\begin{small}
%\begin{tabular}{cccc}\\\toprule  
%Network& DGeoLIP & NGeoLIP & LiPopt \\\midrule
%2-layer/8 units    & 23.1& 21.5& 1533\\  %\midrule
%2-layer/256 units    & 976.0& 70.9& 2690\\  %\midrule
%8-layer/64 units    & 329.5& -----& N/A\\  \bottomrule
%\end{tabular}
%\end{small}
%\end{wraptable}

\paragraph{Network setting}We run the experiments on fully-connected feed-forward neural networks, trained with the MNIST dataset for $10$ epochs using the ADAM optimizer~\citep{adam}. All the trained networks have accuracy greater than $92\%$ on the test data. % We will use a list of numbers to denote the network configuration. For example, $(16)$-network denotes a feed-forward neural network with a single hidden layer, and $16$ ReLU units; $(64, 64, 16)$-network denotes a feed-forward neural network with $3$ hidden layers, and $64$, $64$, $16$ ReLU units, in each layer respectively.
For two-layer networks, we use $8$, $16$, $64$, $128$, $256$ hidden nodes. For multiple-layer networks, we consider $3$, $7$, $8$-layer networks, and each hidden layer has $64$ $\relu$ units.
Because MNIST has $10$ classes, we report the estimated FGL with respect to label $8$ as in~\citet{lipopt}, and the average running time per class: we record the total computation time for all 10 classes from each tool, and report the average time per class.  % for the computation.

%However, for degree $3$, each Lipschitz constant estimation is projected to finish after at least $200$ hours for the neural network we are working on. Therefore, we are not able to run the experiments for multiple hidden layer networks. In our experiments, we are only able to run LiPopt-2 for two layer networks.

%\paragraph{Toy Network}We will first consider a toy example: a single hidden layer neural network with $16$ $\relu$ units. Because we only have $16$ hidden units, we can run a brute-force search to find the maximum value of the integer programming problem. The results are summarized in~\cref{tb:toyLC,tb:toyTime}.

\subsection{Discussion}\label{sec:dis}
%In the following, we will discuss how our evaluation addresses the research questions, and their implications. We first address scalability, and then precision. In the end, we discuss duality and the related quantitative geometric perspective .

%\subsection{Scalability}
We present selected results in~\cref{tb:selected,tb:selected-time}, and related major discussions here. The full results, more experimental setup and additional discussions can be found in~\cref{sec:more_eval}.

\paragraph{RQ1}In the experiments of LiPopt, we only used LiPopt-2. In theory, if one can go higher in the LP hierarchy in LiPopt, the result becomes more precise. 
However, in the case of fully-connected neural networks, using degree-3 in LiPopt is already impractical.
For example, on the simplest network that we used, i.e., the single-hidden-layer neural network with $8$ hidden units, using LiPopt-3, one $\ell_\infty$-FGL computation needs at least $200$ hours projected by LiPopt. As a result, for all the LiPopt-related experiments, we were only able to run LiPopt-2.
As~\citet{lipopt} pointed out, the degree has to be at least the depth of the network to compute a valid bound, so we have to use at least LiPopt-k for $k$-layer networks. LiPopt is unable to handle neural networks with more than two layers because this requires LiPopt with degrees beyond 2.
Even if we only consider LiPopt-2 on two-layer networks, the running time is still much higher compared to GeoLIP. % In the meantime, GeoLIP achieves a much lower $\ell_\infty$-FGL estimation on two layer networks. %
This demonstrates the great advantage of GeoLIP in terms of scalability compared with LiPopt.
%
%\paragraph{Sparsity might be a bug, not a feature}\cite{lipopt} argued that LiPopt can exploit the sparsity of neural network, and works better on architectures with higher sparsity. Empirically, the authors examined LiPopt on networks that are trained to encourage sparsity. However, improving sparsity of a neural network may hurt its robustness. \cite{bubeck2021a} showed that for a network to achieve both high accuracy and smoothness, the neural network has to increase its capacity significantly. Increasing sparsity effectively decreases the number of neurons, and therefore may hurt the model's robustness. A thorough investigation of sparsity is necessary when it comes to the robustness of neural networks, and this can be a future work.
%\subsection{Precison}
%\paragraph{Precision of GeoLIP}
If we compare LiPopt-2 with GeoLIP on two-layer networks from~\cref{tb:selected}, it is clear that GeoLIP produces more precise results. For networks with depth greater than $2$, we can only compare GeoLIP with the matrix-norm-product method. As we can see from all experiments, GeoLIP's estimation on the FGL is always much lower than MP. %the matrix-norm product. %, for example, for the $8$-layer network, the FGL measured by the matrix-norm product is around $10000$ times greater than GeoLIP's result. Given that GeoLIP is the only tool that can scale to deep networks other than the matrix-norm product method, this demonstrates the importance of GeoLIP.

We have also shown that the two-layer network $\ell_\infty$-FGL estimation from GeoLIP has a theoretical guarantee with the approximation factor $K_G < 1.783$~(\cref{thm:infty-FGL}). If we compare the two-layer network results from GeoLIP and Sampling in~\cref{tb:selected}, which is a lower bound of true Lipschitz constant, the ratio is within $1.783$. This validates our theoretical claim. 

%More interestingly, we showed that GeoLIP's approximation factor for $\ell_2$-FGL estimation on two layer networks is $\sqrt{\frac{\pi}{2}}\approx 1.253$. The $\ell_2$-FGL is very close to the sampled lower bound of true Lipschitz constant in~\cref{tb:dual2}. On the other hand, because the result from GroLIP is an upper bound of FGL, and this result is not much greater than the sampled lower bound of true Lipschitz constant, this empirically demonstrates that the true Lipschitz constant is not very different from the FGL on two-layer networks. 

\paragraph{RQ2}
In~\cref{sec:dual}, we have demonstrated the duality between NGeoLIP and LipSDP for the $\ell_2$-FGL estimation on two-layer networks, even though the approaches appear drastically different. The experiments show that on two-layer networks, LipSDP and NGeoLIP for $\ell_2$-FGL estimations (\cref{tb:dual2} in the appendix), and DGeoLIP and NGeoLIP for $\ell_\infty$-FGL estimations produce the same values. These results empirically validate the duality arguments, and also all the related SDP programs. %Similarly, the dual program proposed in~\cref{sec:dual_infty} and the natural SDP relaxation of  on two-layer networks generate the same values. These results provide empirical evidence for the correctness of the duality arguments, and the SDP programs themselves. %Though~\citet{anti_lipSDP} showed that the most precise version of LipSDP is invalid for estimating an upper bound of $\ell_2$-FGL, our dual-program argument shows that the less precise version of LipSDP is correct.

%If we compare the running time in~\cref{tb:TLTime,tb:l2Time}, the dual program takes more time to solve than the natural relaxation. This is particularly true when the number of hidden neurons increases. From the reported numbers of variables and equality constraints by CVX, the dual program and natural relaxation have similar numbers. It is also observed that the CPU usage is higher when the natural relaxation is being solved. 
%We want to point out that the running time and optimization algorithm are solver-dependent, and efficiently solving SDP is beyond the scope of this work. It is an interesting future direction to exploit the block structure of the dual programs, and develop algorithms that are compatible with those programs, because training low Lipschitz network is a critical task, and it is promising to incorporate the SDP programs.

%\paragraph{$\ell_2$ versus $\ell_\infty$ FGLs} If we compare results from~\cref{tb:TLLC,tb:dual2}, we can also find that the discrepancy between matrix product method and sampled lower bound is much smaller in the $\ell_2$ case. This could also explain why Gloro works for $\ell_2$-perturbations but not the $\ell_\infty$ case in practice, where~\citet{gloro} used matrix-norm product to compute the Lipschitz constant of the network in Gloro.

\paragraph{SDP relaxation}Applying SDP on intractable combinatorial optimization problem was pioneered by the seminal Goemans-Williamson algorithm for the $\MAXCUT$ problem~\citep{maxcut}. For two-layer networks, we have reduced the FGL estimation to the mixed-norm problem, and provide approximation algorithms with ratios compatible with the known optimal constants in the corresponding mixed-norm problems. Improving them can be a very hard task. We also provide a compositional SDP interpretation of LipSDP-Neuron. Although \citet{anti_lipSDP} demonstrated the flaw in LipSDP-Network, our compositional SDP interpretation shows that LipSDP-Neuron is correct. In fact, from the compositional SDP interpretation, the program is only constrained by the underlying perturbation geometry and the layer-wise restriction from each hidden unit, so the constraints and objective exactly encode the FGL-estimation problem without additional assumptions. Because often the SDP relaxation for intractable problems gives the optimal known algorithms, we conjecture that GeoLIP and LipSDP are also hard to improve on FGL estimations. 

\citet{lipopt} used polynomial optimization to address the $\ell_\infty$-FGL estimation. We argue that approaching FGL-estimations from the perspective of polynomial optimization loses the accurate characterization of this problem. For example, for two-layer networks, we have provided constant approximation algorithms to estimate FGLs in both $\ell_\infty$ and $\ell_2$ cases. However, for a general polynomial optimization problem on a cube, we cannot achieve constant approximation. For example, the maximum independent set of a graph can be encoded as a polynomial optimization problem over a cube~\citep{motzkin_straus_1965}, but the maximum independent set problem cannot be approximated within a constant factor in polynomial time unless $\Pclass = \NP$~\citep{trevisan2004inapproximability}.
\section{Related work}
%We have discussed extensively recent works  on FGL estimations~\citep{lipsdp,lipopt,anti_lipSDP,certified_def}, and compare GeoLIP with them. % In particular,~\citet{certified_def} proposed to incorporate their SDP to train certifiably robust neural networks via the dual program. However, their SDP formulation only works for two-layer networks. Because the SDP programs in our work are already in the dual form, and work for multi-layer networks, it is an immediate future work to train general certifiably robust networks.
%In particular, 
\citet{chen2020semialgebraic} employed polynomial optimization to compute the true Lipschitz constant of ReLU-networks for $\ell_\infty$-perturbations, and used Lasserre's hierarchy of SDPs~\citep{LasserreH} to solve the polynomial optimization problem. However, their approach is highly tailored to ReLU networks, while ours, like LipSDP, can handle common activations, such as sigmoid and ELU.

\citet{lipopt} also proposed to use LiPopt to estimate the local Lipschitz constant. However, estimating this quantity is not the problem studied in our work, and there are tools specifically designed for local perturbations and the Lipschitz constant~\citep{clarkeG,zhang2019recurjac}.

Lipschitz regularization of neural networks is an important task, and recent works~\citep{lipreg1,lipreg2,lipreg3,lipreg4,lipreg5} have investigated this problem. However, here we study a related but different problem, i.e., Lipschitzness measurement of neural networks. Our work can motivate new Lipschitz regularization
techniques.
%The global Lipschitz constant of a neural network is data-independent. In practice, the Lipschitz constant restricted on a specific dataset might have a lower Lipschitz constant, and further improves the certifiable robustness on the dataset. 

%4. Quantitative Geometry 

%\section{Future Work}
%Why different $\ell_p$ smoothness and robustness are different? Quantitative approach instead of existential view.

%How to combine the method to train networks.
\section{Conclusion}
In this work, we have provided a quantitative geometric framework for FGL estimations, and also algorithms for the $\ell_\infty$-FGL estimation. %We hope that this framework can bring new insights and theoretical tools to the study of neural-network smoothness and robustness. 
One important lesson is that when transferring techniques from one perturbations to another ones, we should also transfer the underlying geometry. One future work is to train smooth neural networks using the SDPs proposed in this paper. %In the meantime, we believe that the algorithms from our work can be a crucial step to train smooth neural networks.
\acksection
The authors thank Vijay Bhattiprolu for introducing recent progress on the mixed-norm problems. The work is partially supported by Air Force Grant FA9550-18-1-0166, the National Science Foundation (NSF) Grants CCF-FMitF-1836978, IIS-2008559, SaTC-Frontiers-1804648, CCF-2046710
and CCF-1652140, and ARO grant number W911NF-17-1-0405. Zi Wang and Somesh Jha are
partially supported by the DARPA-GARD problem under agreement number 885000.

\bibliography{paper}
\bibliographystyle{ACM-Reference-Format}
%%%%%%%%%%%%%%%%%%%%%%%%%%%%%%%%%%%%%%%%%%%%%%%%%%%%%%%%%%%%
\section*{Checklist}

\begin{enumerate}

\item For all authors...
\begin{enumerate}
  \item Do the main claims made in the abstract and introduction accurately reflect the paper's contributions and scope?
    \answerYes{} Our main claims in the abstract and introduction accurately reflect the paper's contributions and scope. In particular, we listed our contributions in~\cref{sec:intro}, and forward referenced each contribution to the corresponding section in the paper.
  \item Did you describe the limitations of your work?
    \answerYes{} We clearly defined what is the quantity to measure and what the network is in~\cref{sec:prelim}, and what the assumptions are for each theorem in~\cref{sec:two-layer}.
  \item Did you discuss any potential negative societal impacts of your work?
    \answerYes{} We discussed them in~\cref{sec:societal}.
  \item Have you read the ethics review guidelines and ensured that your paper conforms to them?
    \answerYes{} We have read the ethics review guidelines. Because our work is to measure the smoothness of neural networks (see~\cref{sec:intro,sec:prelim}), and we only used the standard MNIST dataset (see~\cref{sec:ex-design}), the paper conforms to guidelines. 
\end{enumerate}

\item If you are including theoretical results...
\begin{enumerate}
  \item Did you state the full set of assumptions of all theoretical results?
    \answerYes{} We clearly defined the quantity to measure and the network structures in~\cref{sec:prelim,sec:two-layer,sec:dual,sec:dual_infty}.
        \item Did you include complete proofs of all theoretical results?
    \answerYes{} We provided important intuition and ideas in the main paper (see~\cref{sec:two-layer,sec:dual,sec:dual_infty}), and included complete proofs in~\cref{sec:app_proof}.
\end{enumerate}

\item If you ran experiments...
\begin{enumerate}
  \item Did you include the code, data, and instructions needed to reproduce the main experimental results (either in the supplemental material or as a URL)?
    \answerYes{} We included the code, data, and instructions as a URL.
  \item Did you specify all the training details (e.g., data splits, hyperparameters, how they were chosen)?
    \answerYes{} We included major experimental setting in~\cref{sec:ex-design}, and detailed specification in~\cref{sec:ex_spec}.
        \item Did you report error bars (e.g., with respect to the random seed after running experiments multiple times)?
    \answerNA{} Our experiments are deterministic. Given a neural network, our algorithm always returns the same result.
        \item Did you include the total amount of compute and the type of resources used (e.g., type of GPUs, internal cluster, or cloud provider)?
    \answerYes{} This was provided in~\cref{sec:ex_spec}.
\end{enumerate}

\item If you are using existing assets (e.g., code, data, models) or curating/releasing new assets...
\begin{enumerate}
  \item If your work uses existing assets, did you cite the creators?
    \answerYes{} In terms of dataset, we only used the standard MNIST, and cited the creators. All the tools used in the paper were properly cited. See~\cref{sec:ex-design,sec:ex_spec}.
  \item Did you mention the license of the assets?
    \answerYes{} See~\cref{sec:ex_spec}.
  \item Did you include any new assets either in the supplemental material or as a URL?
    \answerYes{} We included our code as a URL.
  \item Did you discuss whether and how consent was obtained from people whose data you're using/curating?
    \answerNA{} 
  \item Did you discuss whether the data you are using/curating contains personally identifiable information or offensive content?
    \answerNA{} 
\end{enumerate}

\item If you used crowdsourcing or conducted research with human subjects...
\begin{enumerate}
  \item Did you include the full text of instructions given to participants and screenshots, if applicable?
    \answerNA{} 
  \item Did you describe any potential participant risks, with links to Institutional Review Board (IRB) approvals, if applicable?
    \answerNA{} 
  \item Did you include the estimated hourly wage paid to participants and the total amount spent on participant compensation?
    \answerNA{} 
\end{enumerate}

\end{enumerate}
%%%%%%%%%%%%%%%%%%%%%%%%%%%%%%%%%%%%%%%%%%%%%%%%%%%%%%%%%%%%
\appendix

\section{Elided background, derivations and proofs}\label{sec:app_proof}
\subsection{Additional analysis background}\label{sec:analysis}
\paragraph{Gradient as operator}If a function $g: \R^m\rightarrow \R$ is a differentiable function at $a\in \R^m$, then the total derivative of $g$ at $a$ is
\[
    Dg(a) = [\frac{\partial g}{\partial x_1}(a), \ldots, \frac{\partial g}{\partial x_m}(a)],
\]
and the gradient of $g$ at $a$ is $\nabla g(a)$ is the transpose of $Dg(a)$. The linear approximation of $g$ at $a$ is $\inprod{Dg(a), dx}$. Equivalently, we can view the change of a function with respect to an infinitesimal perturbation as the inner product of $\nabla g(a)$ and $dx$. In this sense, the gradient acts as an operator on the perturbation. 

\paragraph{Differentiable activation}Because we want to upper bound the true Lipschitz constant, we only need to show that the quantity considered in the paper indeed upper bounds the true Lipschitz constant considered in the paper. If the activation function is differentiable, then the neural network $f$ is also differentiable, so~\cref{eq:FGL} is trivially true, as proved and applied in~\citet[Theorem 1 and Equation 4]{lipopt}.

\paragraph{ReLU activation}For ReLU networks, it is true if we have $[a, b] = [0, 1]$. One can consider the (Clarke) generalized Jacobian as in~\citet{jordan2021exactly}. At each input point, the Clarke Jacobian is contained in $\{(W^1)^T \cdot \ACTL(v^2) \cdot\cdots\cdot \ACTL(v^d) (W^{d})^T \mid v^i\in[a,b]^{n_i}\}$. Alternatively, we can also use the perturbation propagation argument in~\cref{sec:dual} to see this upper bound. Note that~\citet{certified_def} used this interval representation for ReLU's derivative.

\paragraph{Maximum over hypercube} Now we want to show that the optimization problems over hypercubes considered in this work attain the maximum at the vertices. Without of loss of generality, let us assume the hypercube is $[-1, 1]^n$. Otherwise, we can transform the hypercube to $[-1, 1]^n$.
Let $A\in \R^{m\times n}$, $x\in \R^n$, $y\in \R^m$, and $z\in \R^n$.

We will use the following facts\begin{enumerate}
    \item $\norm{x}_1 = \max_{z\in\{-1, 1\}^n}\inprod{x, z}$;
    \item $\ell_\infty$ is the dual of $\ell_1$;
    \item Let $U\subseteq \R^n$. When $\max_{x, z\in U}\inprod{Ax, Az}$ is well-defined, we have $\max_{x\in U}\inprod{Ax, Ax} = \max_{x, z\in U}\inprod{Ax, Az}$.
\end{enumerate}
The first fact is from $\norm{x}_1 = |x_1|+\ldots +|x_n| = \max_{z\in\{-1, 1\}^n}\inprod{x, z}$. The second fact is from Hölder's inequality for finite-dimensional vector space. For the third one, $\inprod{Ax, Az}$ is maximized only when $Ax = Az$. Now we can show that the maximization problems considered in this paper attain the maximum at the hypercube vertices.
\begin{align}\label{eq:vertex_1}
\begin{split}
\max_{\norm{x}_\infty=1}\norm{Ax}_1 &= \max_{\norm{x}_\infty=1, y\in\{-1, 1\}^m}\inprod{Ax, y} \quad\quad \text{(From fact (1))} \\&= \max_{\norm{x}_\infty=1, y\in\{-1, 1\}^m}\inprod{x, A^Ty} \\&= \max_{y\in\{-1, 1\}^m}\norm{A^Ty}_1 \quad\quad\quad\quad\;\;\; \text{(From fact (2))}\\&=  \max_{x\in\{-1, 1\}^n, y\in\{-1, 1\}^m}\inprod{x, A^Ty}\\&=  \max_{x\in\{-1, 1\}^n, y\in\{-1, 1\}^m}\inprod{Ax, y}.
\end{split}
\end{align}

\begin{align*}
\begin{split}
\max_{\norm{x}_\infty=1}\norm{Ax}_2^2 &= \max_{\norm{x}_\infty=1}\inprod{Ax, Ax} \\&= \max_{\norm{x}_\infty=1, \norm{z}_\infty=1}\inprod{Ax, Az}\;\;\;\;\;\;\;\; \text{(From fact (3))}\\&=
\max_{\norm{x}_\infty=1, \norm{z}_\infty=1}\inprod{A^TAx, z}.
\end{split}
\end{align*}

Using the similar idea in~\cref{eq:vertex_1}, we have \[\max_{\norm{x}_\infty=1}\norm{Ax}_2^2 = \max_{x\in\{-1,1\}^n, z\in\{-1,1\}^n}\inprod{Ax, Az} = \max_{x\in\{-1,1\}^n}\inprod{Ax, Ax}.\]
More generally, in the bilinear forms considered above, if $x = x_1\otimes\cdots\otimes x_d$ is generated by the tensor product of variables over cubes, we can fix one variable and write $x$ as a matrix product, and then move the fixed variable to the hypercube vertices. We can repeat this process to move all variables to the vertices. 

\subsection{Additional definitions}
For any neural network $f$, let $OPT(f)$ be the optimal value of~\cref{eq:FGL}. We say an algorithm $\mathcal{A}$ is an approximation algorithm for~\cref{eq:FGL} with approximation ratio $\alpha > 1$, if $OPT(f)\leq \mathcal{A}(f) \leq \alpha OPT(f)$.

\subsection{SDP for the $\infty\rightarrow 1$ mixed-norm problem}\label{sec:mixed-norm}
Recall that for $v\in \R^m$, $\norm{v}_1 = \max_{\norm{u}_\infty = 1}\inprod{u, v}$. We can reformulate the mixed-norm problem as follows:
\[
    \max_{x\in \{-1, 1\}^n} \norm{Ax}_1 = \max_{(x,y)\in \{-1, 1\}^{n+m}} \inprod{Ax, y}.
\]

If we let $z = \begin{pmatrix}
x^T & y^T
\end{pmatrix}$, we can have
\[
\max_{(x,y)\in \{-1, 1\}^{n+m}} \inprod{Ax, y} = \max_{z\in \{-1, 1\}^{n+m}} z\cdot B\cdot z^T,
\]
where $B$ is a $(m+n)\times (m+n)$ matrix. The last $m$ rows and first $n$ columns of $B$ is $A$, and the rest are $0$: $B = \begin{pmatrix}
0 & 0 \\
A & 0
\end{pmatrix}$.

The natural SDP relaxation of the $\infty\rightarrow 1$ mixed-norm problem is:
\begin{align*}
    \max \trace(BX)\\ 
    s.t.\;  X\succeq 0, X_{ii}=1, & i\in [n+m].
\end{align*}
In other words, we treat $X$ as the SDP matrix relaxed from the rank-1 matrix $z^T\cdot z$.

\subsection{SDP relaxation and Grothendieck inequalities}\label{sec:grothSDP}
In this work, we used the Grothendieck inequality as in~\cref{thm:groth-local}:
\begin{equation}\label{eq:groth}
    \max_{u_i, v_j\in B(H)}\sum_{i,j} A_{ij}\inprod{u_i, v_j}_H \leq K_G\norm{A}_{\infty\rightarrow 1},
\end{equation}
for any $A\in \R^{n\times m}$; and the little Grothendieck inequality:
\begin{equation}\label{eq:little-groth}
  \max_{u_i, v_j\in B(H)}\sum_{i,j} (A^TA)_{ij}\inprod{u_i, v_j}_H \leq \frac{\pi}{2}\norm{A}^2_{\infty\rightarrow 2}.  
\end{equation}
Notice that $A^TA$ is a PSD matrix.

As discussed in~\cref{sec:mixed-norm},
\[
\norm{A}_{\infty\rightarrow 1} = \max_{z\in \{-1, 1\}^{n+m}} z\cdot B\cdot z^T.
\]
The natural SDP relaxation is
\begin{align*}
    \max \trace(BX)\\ 
    s.t.\;  X\succeq 0, X_{ii}=1, & i\in [n+m].
\end{align*}
Because $X\succeq 0$, $X = MM^T$ for some $M\in\R^{(m+n)\times d}$, where $d\geq 1$. Let $M_i$ be the $i$-th row vector of $M$. $X_{ij} = \inprod{M_i,M_j}$, and $X_{ii}=1$ means $\inprod{M_i, M_i} = 1$. As a result, $\trace(BX) = \sum_{i,j} A_{ij}X_{ij} = \sum_{i,j} A_{ij}\inprod{M_i, M_j}_H$, where $H$ is the Hilbert space of $\R^d$ equipped with the canonical inner product. Thus, \cref{eq:groth} implies that $K_G$ is the approximation ratio in the SDP relaxation for the mixed-norm problem.

In contrast, in the mixed-norm problem, the variable to $B_{ij}$ is $z_iz_j$, the product of two scalars. 
If $d=1$ in the SDP relaxation, $M$ is a column vector, and $X$ is a rank-1 matrix. In this case, the SDP coincides with the combinatorial problem, because the inner product degenerates to the multiplication of two scalars. Hence, the SDP relaxation can be viewed as a continuous relaxation of a discrete problem, and \cref{eq:groth} quantifies this geometric transformation. Another interpretation for the SDP relaxation is that SDP drops the rank-1 constraint in the quadratic formulation of the mixed-norm problem.
 
\subsection{Rescaling from 0-1 cube to norm-1 cube}\label{sec:rescale}
Now let us show how we transform the 0-1 cube in~\cref{eq:two_layer_obj} to a norm-1 cube, and formulate an equivalent optimization problem. As a result, we can apply the SDP program in~\cref{eq:sdp_mnp1} to compute an upper bound of the $\ell_\infty$-FGL. 

Let $x_i = (t_i+1)/2$, where $t_i\in\{-1,1\}$. We have
\begin{align}\label{eq:shift}
\begin{split}
    &\max_{x\in \{0, 1\}^n} \norm{Ax}_1 \\= &\max_{x\in \{0, 1\}^n, y\in \{-1, 1\}^m} y^TAx \\= &\max_{t\in \{-1, 1\}^n, y\in \{-1, 1\}^m} \frac{1}{2}y^TA(t+e_n).
\end{split}
\end{align}
Let $OPT_1$ be the optimal value of
\[
    \max_{(t,y)\in \{-1, 1\}^{n+m}} y^TA(t+ e_n).
\]
Introduce another variable $\tau\in\{-1,1\}$, and let $OPT_2$ be the optimal value of
\begin{equation}\label{eq:extra_var}
    \max_{(t,y,\tau)\in \{-1, 1\}^{n+m+1}} y^TA(t+ \tau e_n).
\end{equation}
\begin{lemma}\label{lem:equality}
    $OPT_1=OPT_2$.
\end{lemma}

\begin{proof}
Clearly $OPT_2\geq OPT_1$. 

Now if $(\hat{t}, \hat{y}, \tau=-1)$ is an optimal solution to~\cref{eq:extra_var}, then $(-\hat{t}, -\hat{y}, \tau=1)$ is also an optimal solution, so $OPT_2\leq OPT_1$.
\end{proof}

Now let $z = (t, \tau)$, and we can verify that 
$y^TA(t+ \tau e_n) =  y^T B z$, where
$B = \begin{pmatrix}
A & Ae_n
\end{pmatrix}$.

As a result, the semidefinite program to the $\ell_\infty$-FGL constant is 
\begin{align*}
\begin{split}
    \max \;&\frac{1}{2}\trace(BX)\\ 
    s.t.\;  X \succeq 0, X_{ii}=1&, i\in [n+m+1],
\end{split}
\end{align*}
where $B$ is a $(n+1+m)\times (n+1+m)$ matrix, and 
$B = \begin{pmatrix}
0 & 0 & 0\\
A & Ae_n & 0
\end{pmatrix}$.

\subsection{Proof of~\cref{thm:hardness}}
\begin{proof}
We will use the cube rescaling techniques introduced in~\cref{sec:rescale}. \citet{cut-norm} showed that matrix cut-norm is $\MAXSNP$-hard. We will show that if one can solve the FGL estimation problem, then one can find the cut norm of a matrix.

Given a matrix $A$, the cut norm of a matrix $A\in \R^{m\times n}$ is defined as 
\[
    CN(A) = \max_{x\in\{0,1\}^n, y\in \{0,1\}^m} \inprod{Ax, y}.
\]

We need to transform $y$ from 0-1 cube to norm-1 cube, so similarly let $y_i = (t_i+1)/2$, where $t_i\in \{-1,1\}$. The we will have
\[
    CN(A) = \max_{x\in\{0,1\}^n, y\in \{0,1\}^{m}} \inprod{Ax, y} = \frac{1}{2}\max_{x\in\{0,1\}^n, t\in \{-1,1\}^{m}} \inprod{Ax, (t+e_m)}.
\]

Let $B = \begin{pmatrix}
A\\
e_m^TA
\end{pmatrix}$. From above we know that 
\[
\max_{x\in\{0,1\}^n, t\in \{-1,1\}^{m}} \inprod{Ax, (t+e_m)}
= \max_{x\in\{0,1\}^n, (t, \tau)\in \{-1,1\}^{m+1}} \inprod{Bx, (t,\tau)}.
\]

One can then construct a two layer neural network, where the first weight matrix is $B^T$, and the second weight matrix is $(1,\ldots, 1)\in \R^{n}$. Because the network we consider has only one output, the second weight matrix is only a vector. The FGL of this network is exactly twice of the cut norm of $A$.
\end{proof}

\subsection{Proof of~\cref{thm:infty-FGL}}
\begin{proof}
Let $B = \begin{pmatrix}
0 & 0 & 0\\
A & Ae_n & 0
\end{pmatrix}$. Combing~\cref{eq:shift,eq:extra_var,eq:sdp_mnp1}, the approximation algorithm for~\cref{eq:two_layer_obj} where $q=1$ is induced by the following SDP program:
\begin{align*}%\label{eq:inf_sol}
\begin{split}
    \max \;&\frac{1}{2}\trace(BX)\\ 
    s.t.\;  X \succeq 0, X_{ii}=1&, i\in [n+m+1].
\end{split}
\end{align*}
\end{proof}

\subsection{Natural SDP relaxation of $\ell_2$-FGL estimation}\label{sec:l2fgl}
Now let $q=2$ in~\cref{eq:two_layer_obj}, we will have:
\[
    \max_{y\in\{0,1\}^n} \norm{Ay}_2.
\]

In other words, we only need to solve the following program:
\begin{equation}\label{eq:l2_qp}
    \max_{z\in \{0,1\}^n} z^T (A^T A) z.
\end{equation}

Let $M = A^TA$, then $M$ is a PSD matrix. We have demonstrated the scaling techniques in~\cref{sec:rescale}. Let $x\in\{-1, 1\}^{n+1}$, one can verify that
\begin{equation}\label{eq:l2_shift}
    \max_{z\in \{0,1\}^n} z^T M z = \frac{1}{4}\max_{x\in \{-1,1\}^{n+1}} x^T \hat{M} x,
\end{equation}
where $\hat{M} = \begin{pmatrix}
M & Me_n\\
e_n^TM & e_n^TMe_n
\end{pmatrix}$.

It is easy to verify that if $M$ is PSD, $\hat{M}$ is also PSD. Because $M = A^TA$, $\hat{M} = (A, Ae_n)^T\cdot (A, Ae_n)$. Now we can consider the following natural SDP relaxation to $\max_{x\in \{-1,1\}^{n+1}} x^T \hat{M}x$:
\begin{align}\label{eq:nat-2}
\begin{split}
   \max \trace(\hat{M}X)\\
    s.t.\;  X \succeq 0, X_{ii}=1, i\in &[n+1]. 
\end{split}
\end{align}

This SDP relaxation admits a $\frac{\pi}{2}$-approximation factor from~\cref{eq:little-groth}~\cite{Rietz,nest}. 

\subsection{Proof of~\cref{thm:2-FGL}}
\begin{proof}
Let $\hat{M} = \begin{pmatrix}
M & Me_n\\
e_n^TM & e_n^TMe_n
\end{pmatrix}$, where $M = A^T A$. Combining~\cref{eq:l2_qp,eq:l2_shift,eq:nat-2}, the approximation algorithm for~\cref{eq:two_layer_obj} where $q=2$ is induced by the following SDP program:
\begin{align}%\label{eq:l2-SDP}
\begin{split}
   \max \frac{1}{2}\sqrt{\trace(\hat{M}X)}\\
    s.t.\;  X \succeq 0, X_{ii}=1, i\in &[n+1]. 
\end{split}
\end{align}
\end{proof}

\subsection{Comparison with~\citet{certified_def}}\label{sec:compar}
\citet{certified_def} formulated the following SDP to upper bound the $\ell_\infty$-FGL on two-layer neural networks:
\begin{align}\label{eq:cert}
\begin{split}
    \max \frac{1}{4}\trace(CX)\\ 
    s.t.\;  X \succeq 0, X_{ii}=1, i\in&[n+m+1],
\end{split}
\end{align}
where $C$ is a $(m+n+1)\times (m+n+1)$ matrix, and 
$C = \begin{pmatrix}
0 & 0 & A^T\\
0 & 0 & e_n^T A^T\\
A & Ae_n & 0
\end{pmatrix}$.

If we compare~\cref{eq:sdp_inf,eq:cert}, $C = B + B^T$. Because $X$ is symmetric, $\trace(CX) = 2 \trace{(BX)}$. Therefore, \cref{eq:sdp_inf,eq:cert} produce the same result.

\subsection{Equivalence between the new optimization program and~\cref{eq:l2_qp}}\label{sec:equiv}

Notice that because $u\in \R^{1\times n}$, $u\Delta \act(x)$ is a scalar. We can view each $z_i$ in~\cref{eq:l2_qp} as $\frac{\Delta \act(x)_i}{\Delta y_i}$, the derivative of $\act(x)_i$ without the limit. Therefore, $\Delta \act(x)_i = z_i \Delta y_i$. Recall that from~\cref{sec:dual}, $\Delta y_i = w_i\Delta x$, so $\Delta \act(x)_i = w_i z_i \Delta x$, then $u\Delta\act(x) = \Delta x\sum_i^n u_iz_i w_i = \Delta x (Az)$, where $A = W^T\ACTL(u)$ as defined in~\cref{eq:two_layer_obj}. 

As a result, from Cauchy–Schwarz inequality, the above objective is
\begin{align*}
\begin{split}
    \max_{\Delta x,\Delta \act(x)} \frac{(u\Delta \act(x))^2}{(\Delta x)^2} &= \max_{z} (Az)^2,\\ 
    s.t.\;  z \in &[a,b]^n.
\end{split}
\end{align*}

This demonstrates the equivalence between the new optimization program and~\cref{eq:l2_qp} when $[a, b] = [0,1]$ for $\act = \relu$. 

%An immediate result is the  Nesterov $\pi/2$ theorem, which can be viewed as a generalization the seminal Goemans-Williamson algorithm. This result enables a tight SDP relaxation for $\ell_2$ FGL constant.

%Recall that the $\ell_2$ FGL is 

%\[
%\max_{Y\in\{0, 1\}^n} \norm{A Y}_2.
%\]

%This is equivalent to optimizing 

%\[
%    \max_{Y\in\{0, 1\}^n} \norm{A Y}_2^2 = \max_{Y\in\{0, 1\}^n} Y^T (A^TA) Y.
%\]

%If 

\section{Polynomial optimization approach to the FGL estimation}\label{sec:poly}
We briefly discuss the gradient approach to estimate the FGL. Let us use a three layer network as an example:
\begin{equation*}%\label{eq:3net}
f(x)= u  \act(V\act(Wx + b_1) + b_2),
\end{equation*}
where $x\in \R^{l\times 1}$, $W\in \R^{n\times l}$, $b_1\in \R^n$, $V\in \R^{m\times n}$, $b_2\in \R^m$ and $u\in \R^{1\times m}$ . 

%of~\cref{eq:3net} as an example, and the multi-layer version follows easily. 

The formal gradient vector of this network is
\[
W^T \ACTL(y)V^T\ACTL(z) u^T,
\]
where $\ACTL(y)\in \R^{n\times n}$ and $\ACTL(z)\in \R^{m\times m}$.
The $i$-th component of this vector is then
\[
    \sum_{k=1}^m\sum_{j=1}^n (u_{k}V_{kj} W_{ji})\cdot(y_jz_k).
\]

Therefore, the $\ell_p$-norm estimation of the formal gradient ends up being a polynomial optimization problem over a cube. For example, the $\ell_1$-norm (corresponding to $\ell_\infty$-perturbations) of the gradient is
\begin{equation}\label{eq:tensor-cut-norm}
\max_{x_i\in\{-1, 1\}, y_j\in\{0, 1\}, z_k\in\{0, 1\}} \sum_{i,j,k=1}^{l,n,m} T_{ijk}\cdot x_iy_jz_k,
\end{equation}
where $T_{ijk} = W_{ji}V_{kj}u_{k}$.

This is essentially a tensor cut-norm problem, and it is an open problem whether there exists an approximation algorithm within a constant factor to the general tensor cut-norm problem~\citep{tensor-cut-norm}. Notice that~\cref{eq:tensor-cut-norm} is not a general tensor-cut-norm problem, because the tensor is generated from the weight matrices. For example, if we fix $j$, the projected matrices of $T$ are of rank-1. Each vector in $T_{:,j,:}$ is the product of $V_{kj}u_{k}$ with the vector $W_{j:}$:
\[
    \forall k: T_{:,j,k} = W_{j:}V_{kj}u_{k}.
\]

However, we do not have the theoretical technique to exploit the low-rank structure of these special polynomial optimization problems. The perturbation analysis in~\cref{sec:dual,sec:dual_infty} can be viewed as exploiting this structure in practice.

%Now let's consider the dual program. Still introduce $\objv$ such that

%\begin{algorithm}[tb]
%   \caption{GeoLIP for Multiple Layer Networks}
%   \label{alg:GeoLIP}
%\begin{algorithmic}
%   \State {\bfseries Input:} A list of $d$ weight matrices $W_i$
%   \If{$d$ is even}
%        \State Group weights into pairs of weights
%    \Else
%        \State Group weights into pairs of weights except for the final three weights
%    \EndIf
%   \State Apply the Jacobian relaxation to all pairs of weights except for the final group to find the Jacobian norm of all pairs of weights
%   \If{$d$ is even}
%        \State Apply the SDP to the final pair of weights to compute the gradient norm
%    \Else
%        \If{Apply tensor relaxation to the final three matrices is feasible}
%        \State Apply tensor relaxation to the final three matrices
%        \Else
%        \State Compute the $\ell_1$-norm of the first matrix, and apply the SDP to the final pair of weights
%        \EndIf
%    \EndIf
%   \State Multiply the result norm from each group together
%\end{algorithmic}
%\end{algorithm}

\section{Complete experimental specifications and results}\label{sec:more_eval}
%\paragraph{Single Hidden Layer}
%We consider two layer neural networks with different numbers of hidden units. 
%The results for two layer networks are summarized in~\cref{tb:TLLC,tb:TLTime}.
GeoLIP is available at \url{https://github.com/z1w/GeoLIP}. To accommodate users who do not have access to MATLAB, we also implement a version based on CVXPY~\citep{diamond2016cvxpy}. However, the MATLAB implementation works more efficiently in terms of memory and speed, and we encourage users to work with the MATLAB version when possible. We conducted all the GeoLIP-related experiments with the MATLAB version.

\subsection{Experimental specifications}\label{sec:ex_spec}
\paragraph{Tools}We obtain the LiPopt implementation from \url{https://github.com/latorrefabian/lipopt}, under the MIT License.

\paragraph{Server specification}All the experiments are run on a workstation with forty-eight Intel\textsuperscript{\textregistered} Xeon\textsuperscript{\textregistered} Silver 4214 CPUs running at 2.20GHz, and 258 GB of memory, and eight Nvidia GeForce RTX 2080 Ti GPUs. Each GPU has 4352 CUDA cores and 11 GB of GDDR6 memory.

\paragraph{Dataset and split}
We used the standard MNIST dataset from the PyTorch package~\citep{pytorch}. We used the ``train'' parameter in the MNIST function to split training and testing data.
\subsection{Experimental results}

\begin{table}[t]
\caption{$\ell_\infty$-FGL estimations of different methods for two-layer networks: DGeoLIP and NGeoLIP induce the same estimations, and they are also close to the sampled lower bounds. In the meantime, the result from GeoLIP is tighter than LiPopt's result.}\label{tb:TLLC}
%\label{sample-table}
\begin{center}
\begin{small}
\begin{sc}
\begin{tabular}{lcccccr}
\toprule
\#Units&DGeoLIP & NGeoLIP & LiPopt-2 & MP & Sample & BruF \\
\midrule
8    & 142.19& 142.19& 180.38& 411.90& 134.76 & 134.76\\
16   &185.18& 185.18& 259.44 & 578.54 & 175.24 & 175.24\\
64    & 287.60&287.60& 510.00& 1207.70 & 253.89 & N/A\\
128    & 346.27&346.27& 780.46& 2004.34 & 266.22 & N/A\\
256    & 425.04&425.04& 1011.65& 2697.38 & 306.98 & N/A\\
%512    & 537.80& 1452.41& 4755.23 &  376.06\\
\bottomrule
\end{tabular}
\end{sc}
\end{small}
\end{center}
\end{table}

\begin{table}[t]
\caption{Average running time (in seconds) of different methods for two-layer-network $\ell_\infty$-FGL estimations: GeoLIPs are faster than LiPopt.}\label{tb:TLTime}
%\label{sample-table}
\begin{center}
\begin{small}
\begin{sc}
\begin{tabular}{lccccr}
\toprule
\# Hidden Units& DGeoLIP & NGeoLIP & LiPopt-2 & BruF \\
\midrule
8    & 23.1& 21.5& 1533 & < 0.1\\
16    & 28.1& 22.3& 1572 & 4.8\\
64    & 93.4& 31.7& 1831 & N/A\\
128    & 292.5& 42.2& 2055 & N/A\\
256    & 976.0& 70.9& 2690 & N/A\\
\bottomrule
\end{tabular}
\end{sc}
\end{small}
\end{center}
\end{table}
\paragraph{Single hidden layer}
We consider the $\ell_\infty$-FGL estimation on two layer neural networks with different numbers of hidden units. The results are summarized in~\cref{tb:TLLC,tb:TLTime}.

%\paragraph{Multiple Hidden Layers}
%We consider $3$, $7$, $8$-layer neural networks. Each hidden layer in the network has $64$ $\relu$ units. %report results from GeoLIP and matrix product method. Only GeoLIP and matrix product method can scale to the networks considered here. Here we consider three networks: $(64, 16)$-network, $(64, 64, 64, 64, 64, 16)$-network and $(64, 64, 64, 64, 64, 64, 16)$-network. 
%The results are summarized.

\begin{table}[t]
\caption{$\ell_\infty$-FGL estimations of different methods for multi-layer networks: GeoLIP's result is much tighter than the matrix-product method, and LiPopt is unable to handle these networks.}\label{tb:MLLC}
%\label{sample-table}
\begin{center}
\begin{small}
\begin{sc}
\begin{tabular}{lccccr}
\toprule
\# Layers&GeoLIP &  Matrix Product & Sample & liptopt \\
\midrule
3    & 529.42& 9023.65 & 311.88 & N/A\\
7    & 5156.5& $1.423*10^7$ & 1168.8& N/A \\
8    & 8327.2& $8.237*10^7$ & 1130.6& N/A\\
\bottomrule
\end{tabular}
\end{sc}
\end{small}
\end{center}
\end{table}

\begin{table}[t]
\caption{Average running time (in seconds) of GeoLIP for multi-layer-network $\ell_\infty$-FGL estimations.}\label{tb:MLTime}
%\label{sample-table}
\begin{center}
\begin{small}
\begin{sc}
\begin{tabular}{lcccr}
\toprule
3-Layer Net& 7-Layer Net & 8-Layer Net \\
\midrule
 120.9& 284.3& 329.5 \\
\bottomrule
\end{tabular}
\end{sc}
\end{small}
\end{center}
\end{table}

\paragraph{Multiple hidden layers}
We consider the $\ell_\infty$-FGL estimation on $3$, $7$, $8$-layer neural networks. Each hidden layer in the network has $64$ $\relu$ units. %report results from GeoLIP and matrix product method. Only GeoLIP and matrix product method can scale to the networks considered here. Here we consider three networks: $(64, 16)$-network, $(64, 64, 64, 64, 64, 16)$-network and $(64, 64, 64, 64, 64, 64, 16)$-network. 
The results are summarized in~\cref{tb:MLLC,tb:MLTime}.

\paragraph{$\ell_2$-FGL estimation}We measure the $\ell_2$-FGL on two-layer networks mainly to compare whether~\cref{eq:l2-SDP} and LipSDP produce the same result. Additionally, we also want to empirically examine the approximation guarantee from~\cref{thm:2-FGL}. Still, we consider networks with $8$, $16$, $64$, $128$, $256$ hidden nodes. The results are summarized in~\cref{tb:dual2,tb:l2Time}.

\begin{table}[t]
\caption{$\ell_2$-FGL estimations of different methods for two-layer networks: LipSDP and NGeoLIP induce the same estimations, and these results are also close to the sampled lower bounds.}\label{tb:dual2}
%\label{sample-table}
\begin{center}
\begin{small}
\begin{sc}
\begin{tabular}{lccccr}
\toprule
\#Units&NGeoLIP & LipSDP & MP & Sample & BruF \\
\midrule
8    & 6.531& 6.531& 11.035& 6.527 & 6.527\\
16  & 8.801& 8.801& 13.936& 8.795 & 8.799\\
64    & 12.573& 12.573& 22.501& 11.901 & N/A\\
128    & 15.205& 15.205& 30.972& 13.030 & N/A\\
256    & 18.590& 18.590& 35.716& 14.610 & N/A\\
\bottomrule
\end{tabular}
\end{sc}
\end{small}
\end{center}
\end{table}

\begin{table}[t]
\caption{Average running time (in seconds) of LipSDP and NGeoLIP for two-layer-network $\ell_2$-FGL estimations.}\label{tb:l2Time}
%\label{sample-table}
\begin{center}
\begin{small}
\begin{sc}
\begin{tabular}{lccccr}
\toprule
\# Hidden Units& LipSDP & NGeoLIP & BruF  \\
\midrule
8    & 11.5& 1.2 & < 0.1\\
16    & 15.7& 1.2 & 5.1\\
64    & 64.2& 1.3 & N/A\\
128    & 216.1& 1.7 & N/A\\
256    & 758.1& 4.1 & N/A\\
\bottomrule
\end{tabular}
\end{sc}
\end{small}
\end{center}
\end{table}

\subsection{Additional discussions}
\paragraph{Duality} The results in~\cref{tb:dual2} show that the results of LipSDP and GeoLIP on two-layer-network $\ell_2$-FGL estimation are exactly the same, which empirically demonstrates the duality between LipSDP and GeoLIP, as discussed in~\cref{sec:dual}. Though~\citet{anti_lipSDP} showed that the most precise version of LipSDP is invalid for estimating an upper bound of $\ell_2$-FGL, our dual-program argument shows that the less precise version of LipSDP is correct.

\paragraph{Precision}We showed that GeoLIP's approximation factor for the $\ell_2$-FGL estimation on two layer networks is $\sqrt{\frac{\pi}{2}}\approx 1.253$ in~\cref{thm:2-FGL}. The $\ell_2$-FGL from GeoLIP is very close to the sampled lower bound of true Lipschitz constant in~\cref{tb:dual2}. On the other hand, because the result from GeoLIP is an upper bound of FGL, and this result is not much greater than the sampled lower bound of true Lipschitz constant, this empirically demonstrates that the true Lipschitz constant is not very different from the FGL on two-layer networks. 

\paragraph{Running time}If we compare the running time in~\cref{tb:TLTime,tb:l2Time}, the dual program takes more time to solve than the natural relaxation. This is particularly true when the number of hidden neurons increases. From the reported numbers of variables and equality constraints by CVX, the dual program and natural relaxation have similar numbers. It is also observed that the CPU usage is higher when the natural relaxation is being solved. 
We want to point out that the running time and optimization algorithm are solver-dependent, and efficiently solving SDP is beyond the scope of this work. It is an interesting future direction to exploit the block structure of the dual programs, and develop algorithms that are compatible with those programs, because training smooth networks is a critical task, and it is promising to incorporate the SDP programs.

\paragraph{$\ell_2$ versus $\ell_\infty$ FGLs} If we compare results from~\cref{tb:TLLC,tb:dual2}, we can also find that the discrepancy between matrix product method and sampled lower bound is much smaller in the $\ell_2$ case. This could also explain why Gloro works for $\ell_2$-perturbations but not the $\ell_\infty$ case in practice, where~\citet{gloro} used matrix-norm product to upper bound the Lipschitz constant of the network in Gloro.

\paragraph{Sampling} Sampling can only give a lower bound of the true Lipschitz constant, while we are trying to estimate an upper bound. We use sampling as a sanity check to ensure that the SDP method is at least sound and indeed provides an upper bound of the FGL. It is interesting to see that in networks where we can brute-force enumerate all the activation patterns, sampling provides very close results to the ground-truth ones. Notice that for those networks, there are only a few hidden units (8 or 16), while we sample many (200,000) inputs, which might activate all or most of the patterns. However, for networks with many activation nodes, it is infeasible to have a brute-force enumeration of all the activation patterns, so we do not have the ground-truth information. Sampling has no guarantee whether it can activate all patterns unless we have sampled all possible inputs, which is also impractical.

\paragraph{Multi-layer network guarantees}The discrepancy between the results from sampling and GeoLIP is relatively large for multi-layer networks. The approximation guarantee of GeoLIP is in terms of the FGL, rather than true Lipschitz constant. It is unclear how large the gap between true Lipschitz constant and the FGL is for multi-layer networks. Narrowing this gap is an interesting research direction and beyond the scope of this work. We do not know whether for multi-layer networks, GeoLIP has an approximation guarantee that is independent of the network. We leave this as an open problem.

\section{Negative societal impacts}\label{sec:societal}
Our work is mainly theoretical and to measure an intrinsic mathematical property of neural networks, and can benefit the verification of deep-learning systems. A misuse of our work can give a false sense of safety, so the practical use of our work should be careful.

\end{document}